\documentclass[twoside,11pt]{article}
\usepackage{jair, theapa, rawfonts}

\usepackage[implicit=false]{hyperref}

\usepackage{amsmath,amsfonts,bm}

\def\eqref#1{equation~\ref{#1}}

\def\1{\bm{1}}

\def\rvw{{\mathbf{w}}}

\def\rmE{{\mathbf{E}}}

\def\vm{{\bm{m}}}

\DeclareMathAlphabet{\mathsfit}{\encodingdefault}{\sfdefault}{m}{sl}
\SetMathAlphabet{\mathsfit}{bold}{\encodingdefault}{\sfdefault}{bx}{n}

\def\gA{{\mathcal{A}}}

\def\gC{{\mathcal{C}}}

\def\gG{{\mathcal{G}}}

\def\gJ{{\mathcal{J}}}

\def\gM{{\mathcal{M}}}
\def\gN{{\mathcal{N}}}

\def\gS{{\mathcal{S}}}

\DeclareMathOperator*{\argmax}{arg\,max}

\usepackage[T1]{fontenc}
\usepackage{subfiles}
\usepackage{microtype}
\usepackage{graphicx}
\usepackage{subfigure}
\usepackage{booktabs} 
\usepackage{algorithm}
\usepackage{algorithmic}
\usepackage{xcolor}

\usepackage{amsmath}
\usepackage{amssymb}
\usepackage{mathtools}
\usepackage{amsthm}
\usepackage{stfloats}
\usepackage{enumitem}
\usepackage{float}
\usepackage{array,booktabs,makecell,multirow}
\newcommand{\mb}[1]{\mathbf{#1}}
\newcommand{\mc}[1]{\mathcal{#1}}
\renewcommand{\algorithmicrequire}{\textbf{Input:}}
\renewcommand{\algorithmicensure}{\textbf{Output:}}
\theoremstyle{plain}
\usepackage{thmtools, thm-restate}
\newtheorem{theorem}{Theorem}[section]

\newtheorem{lemma}[theorem]{Lemma}

\theoremstyle{definition}
\newtheorem{definition}[theorem]{Definition}

\theoremstyle{remark}

\newcommand{\framework}{{COLE} framework\xspace}
\usepackage{xspace}
\newcommand{\algo}{$\text{COLE}_\text{SV}$\xspace} 
\newcommand{\algoR}{$\text{COLE}_\text{R}$\xspace} 
\newcommand{\changes}[1]{\textcolor{black}{#1}}
\usepackage{arydshln}

\firstpageno{1}

\begin{document}

\title{Tackling Cooperative Incompatibility for \\ Zero-Shot Human-AI Coordination}

\author{\name Yang Li\footnotemark[1] \email yang.li-4@manchester.ac.uk \\
       \addr  Department of Computer Science, The University of Manchester, Manchester, UK
       \AND
       \name Shao Zhang\footnotemark[1] \email shaozhang@sjtu.edu.cn \\
       \name Jichen Sun \email sunjichen@sjtu.edu.cn \\
       \name Wenhao Zhang \email wenhao\_zhang@sjtu.edu.cn \\
       \addr Shanghai Jiao Tong University, Shanghai, China
       \AND
       \name Yali Du \email yali.du@kcl.ac.uk \\
       \addr Department of Informatics, King's College London, London, UK
       \AND
       \name Ying Wen\footnotemark[2]  \email ying.wen@sjtu.edu.cn \\
       \name Xinbing Wang \email xwang8@sjtu.edu.cn  \\
       \addr Shanghai Jiao Tong University, Shanghai, China
       \AND
       \name Wei Pan\footnotemark[2] \email wei.pan@manchester.ac.uk \\
       \addr Department of Computer Science, The University of Manchester, Manchester, UK}


\maketitle
\renewcommand{\thefootnote}{\fnsymbol{footnote}} 
\footnotetext[1]{Equal contribution.} 
\footnotetext[2]{Corresponding authors.}

\begin{abstract}
\changes{
Securing coordination between AI agent and teammates (human players or AI agents) in contexts involving unfamiliar humans continues to pose a significant challenge in Zero-Shot Coordination. 
The issue of cooperative incompatibility becomes particularly prominent when an AI agent is unsuccessful in synchronizing with certain previously unknown partners.} Traditional algorithms have aimed to collaborate with partners by optimizing fixed objectives within a population, fostering diversity in strategies and behaviors. However, these techniques may lead to learning loss and an inability to cooperate with specific strategies within the population, \changes{a phenomenon named cooperative incompatibility in learning. 
In order to solve cooperative incompatibility in learning and effectively address the problem in the context of ZSC}, we introduce the \textbf{C}ooperative \textbf{O}pen-ended \textbf{LE}arning (\textbf{COLE}) framework, which formulates open-ended objectives in cooperative games with two players using perspectives of graph theory to evaluate and pinpoint the cooperative capacity of each strategy. 
\changes{We present two practical algorithms, specifically \algo and \algoR, which incorporate insights from game theory and graph theory.}
We also show that COLE could effectively overcome the cooperative incompatibility from theoretical and empirical analysis. Subsequently, we created an online Overcooked human-AI experiment platform, the COLE platform, which enables easy customization of questionnaires, model weights, and other aspects. Utilizing the COLE platform, we enlist 130 participants for human experiments. Our findings reveal a preference for our approach over state-of-the-art methods using a variety of subjective metrics. Moreover, objective experimental outcomes in the Overcooked game environment indicate that our method surpasses existing ones when coordinating with previously unencountered AI agents and the human proxy model. Our code and demo are publicly available at \url{https://sites.google.com/view/cole-2023}. 
\end{abstract}

\section{Introduction}
Significant advancements in artificial intelligence (AI) research have led to groundbreaking solutions such as AlphaZero~\cite{alphazero} and Segment Everything~\cite{kirillov2023segment}, showing their potential to outperform humans in various tasks. 
\changes{
However, establishing efficient collaboration between AI and unseen partners (either human players or AI agents), a concept known as zero-shot coordination, continues to pose a significant challenge~\cite{Legg2007Universal,Hu2020OtherPlayFZ,de2022zero}.
}
The significance of zero-shot human-AI coordination becomes evident in various real-world applications, including manufacturing~\cite{LI2023102510}, autonomous vehicles~\cite{aoki2021human}, and assistant robots~\cite{Berardinis2020atyour}.
\changes{
The issue of \textit{``cooperative incompatibility''} becomes particularly prominent when an AI agent is unsuccessful in synchronizing with certain previously unknown partners.
As exemplified by multiplayer video games like Honor of King~\cite{hua2022honor} and Overcooked~\cite{HARL}, the constant requirement for players to collaborate with unseen partners is evident. However, AI agents, which are trained via the maximization of rewards, tend to exhibit deterministic strategies and differ from those of human players or other AI agents, leading to a failure in coordinating with some unseen players~\cite{Deheng2020Towards,Yiming2023Towards,HSP}.
}

One of the conventional methods to solve zero-shot human-AI coordination is self-play~\cite{SP}, which involves an iterative strategy refinement process through self-competition. 
Although SP can achieve equilibrium in a game~\cite{Fudenberg1998Theory}, strategies often develop specific behaviors and conventions to secure higher payoffs~\cite{Hu2020OtherPlayFZ}. 
As a result, a fully-converged SP strategy may face difficulties adapting to coordination with previously unencountered strategies and humans~\cite{Adam2018Learning,Hu2020OtherPlayFZ,MEP}. 
To address SP's limitations, current zero-shot coordination (ZSC) approaches concentrate on enhancing strategic or behavioral diversity by incorporating population-based training (PBT) to improve adaptability~\cite{HARL,Canaan2022Hanabi,MEP,TrajDi,charakorn2023generating}. 
PBT aims to boost cooperative performance with other strategies in the population, fostering zero-shot coordination with unknown strategies, which is achieved by maintaining a set of strategies to disrupt SP conventions~\cite{SP} and optimizing the rewards for each pair within the population. 
Most state-of-the-art (SOTA) methods focus on pre-training diverse populations~\cite{FCP,TrajDi} or introducing handcrafted techniques~\cite{Canaan2022Hanabi,MEP} to excel at cooperative games by optimizing fixed objectives within the population. 
And a recent work LIPO~\cite{charakorn2023generating} pursues solutions compatible with their partner agents but incompatible with others in the population.
These approaches have effectively tackled complex cooperative tasks, such as Overcooked~\cite{HARL} and Hanabi~\cite{hanabi2020Nolann}.

However, optimizing a fixed population-level objective, such as maximizing expected rewards within the population~\cite{FCP,TrajDi,MEP} or maximizing self-reward, but minimizing the reward with other agents~\cite{charakorn2023generating}, does not guarantee improved coordination capabilities for strategies within the population. 
Specifically, although overall performance may improve, simultaneous promotion of coordination abilities within the population may not occur. 
This phenomenon, which we refer to as ``\textit{cooperative incompatibility in learning}'', underscores the need to carefully weigh the trade-offs between overall performance and coordination ability when optimizing a fixed population-level objective.

To more effectively describe and formulate the cooperative incompatibility problem, we introduce Graphic-Form Games (GFGs) to reframe cooperative tasks from the perspective of norm-form games.
In a GFG, strategies are depicted as nodes, with edge weights between nodes representing the corresponding cooperative utility of the connected strategies.
Additionally, we derive Preference GFGs (P-GFGs) to profile each node's preferred counterpart in the population, where ``prefer'' signifies that a node achieves a higher score when cooperating with the preferred node rather than its neighbors.
Using (P-)GFGs, cooperative incompatibility in a learning algorithm can be assessed based on the extent to which others prefer the updated strategy.

To address cooperative incompatibility, we propose the Cooperative Open-ended LEarning (\textbf{COLE}) framework, which iteratively generates a new strategy that approximates the best response to empirical gamescapes within P-GFGs. 
We have shown that \framework converges to the local best-preferred strategy at a Q-sublinear rate when using in-degree centrality as the preference evaluation metric. 
\changes{
We propose two practical algorithms, \algo and \algoR, both of which comprise of a simulator, a solver, and a trainer that are incorporated to excel in two-player cooperative game Overcooked~\cite{HARL}. 
The primary distinguishing factor between \algo and \algoR is the solver component. 
While \algo employs the intuitive concept of the Shapley value to evaluate the adaptability of strategies and ascertain the cooperative incompatibility distribution, \algoR takes a different approach. It supplants the Shapley value with the average payoffs associated with nodes in the population as a measure of cooperative capability.}
The trainer aims to approximate the best responses to the cooperative incompatibility distribution mixture found in the population.

\changes{
The prevailing literature of zero-shot human-AI collaboration largely concentrates on the quantitative assessment, indicated in the body of works~\cite{HARL,FCP,HSP}, and often prioritizing metrics such as episode rewards and frequencies. 
Some methods conducted human-AI experiments to verify the zero-shot coordination performance with human players.
Notwithstanding, these investigations tend to scantily scratch the surface of more profound aspects, predominantly bypassing detailed aspects of the subject like preference of human~\cite{Xingzhou2023PECAN,FCP}. 
These oversights accentuate the existing deficiencies in the contemporary experimental environments, which necessitates the urgency to introduce our proposed experimental framework.
For the robust evaluation of human-AI collaboration performance, we devised an online Overcooked human-AI experimental pipeline, which is a cooperative task environment for two players as referenced in \cite{HARL}. 
Our evaluation approach incorporates a broad spectrum of subjective metrics that assess individual games involving AI agents while implementing a comprehensive comparison along the entire participation. Primarily, these encompassing metrics appraise aspects such as intention, contribution, and teamwork during collaboration with AI agents. Besides, they extend the evaluation to encompass fluency, preference, and understanding across all the collaborated agents.
Moreover, we developed all experimental components including ethic agreements, instructions, questionnaires, model weights into a unified experimental platform which we dubbed as the COLE-Platform.
To the best of our knowledge, our open-source human-AI experiment pipeline\footnote{\url{https://github.com/liyang619/COLE-Platform/}} is the first comprehensive human-AI experimentation pipeline for zero-shot human-AI coordination evaluation including turnkey experimental procedures and scale design. 
}

In this paper, we present the findings of a human-agent study involving 130 participants conducted on the COLE platform to assess the performance of our proposed algorithm. 
Participants were tasked with evaluating the AI agent's performance based on three criteria: the human's comprehension of AI intentions, AI's contribution to the task, and the effectiveness of human-AI collaboration. 
Furthermore, each participant compared the COLE platform with another randomly assigned AI agent, ranking the two agents in collaborative fluency, personal preference, and mutual understanding. Our findings demonstrate a clear preference for our algorithm over baseline approaches in various evaluation metrics.
In addition to objective evaluations, we also examined zero-shot human-AI coordination performance by testing our algorithm with previously unencountered baseline agents and a human proxy model. This approach helps to further demonstrate the adaptability and effectiveness of our algorithm in novel collaborative scenarios.
The human proxy model is a widely used behavior cloning model~\cite{HARL}. 
The results of the experiment demonstrate that \algo surpasses the recent SOTA methods in both evaluation protocols. 
Furthermore, analysis of GFGs and P-GFGs during the learning process of \algo reveals that the framework effectively overcomes cooperative incompatibility. 

This work represents an extension of our conference paper \cite{li2023cooperative}. \changes{Significant enhancements incorporated in this study include extending the concept of cooperative incompatibility to include zero-shot human-AI coordination, introducing a comprehensive human-AI experimental pipeline, providing new theoretical analysis to investigate COLE, proposing an extra practical algorithm, and executing a broader set of experiments.}
The primary contributions in this paper can be summarized as follows.
\begin{itemize}
    \item We introduce Graphic-Form Games (GFGs) and Preference Graphic-Form Games (P-GFGs) to intuitively reformulate cooperative tasks, which allows for a more efficient evaluation and identification of cooperative incompatibility during learning.
    \item We propose graphic-form gamescapes to help understand the objective and present the \framework to iteratively approximate the best responses preferred by most others. We prove that the algorithm will converge to the local best-preferred strategy, and the convergence rate will be Q-sublinear when using in-degree preference centrality.
    \item \changes{To the best of our knowledge, we propose the first comprehensive human-AI experimentation pipeline for zero-shot human-AI coordination evaluation including turnkey experimental procedures and scale design.} And we conducted a human-AI experiment involving 130 participants, and the results demonstrate a preference for our \algo over baselines using various subjective metrics. Additional objective experiments confirm the effectiveness of our proposed algorithm compared to SOTAs.
\end{itemize}

The remainder of this paper is structured as follows: Sections~\ref{sec:related} and~\ref{sec:preli} present related works and preliminaries, while Section~\ref{sec:cole} formalizes GFGs, and related concepts and introduces the COLE framework. \changes{Section~\ref{sec:cole_sv} describes two proposed practical algorithms \algoR and \algo}. The COLE human-AI experiment pipeline is detailed in Section~\ref{sec:cole_platform}, and Section~\ref{sec:exp} outlines the experimental settings and results. Finally, we summarize our findings, limitations, and future works in Section~\ref{sec:con}. Additionally, the Appendix offers supplementary information and in-depth proofs, including visual overviews of the human-AI experiment platform.

\section{Related Works}
\label{sec:related} 

\paragraph{Zero-Shot Human-AI Coordination.} 
The zero-shot human-AI coordination is closely related to zero-shot coordination (ZSC).
ZSC aims to train a strategy to coordinate effectively with unseen partners~\cite{Hu2020OtherPlayFZ}. 
Self-play~\cite{SP,HARL} is a traditional method of training a cooperative strategy, which involves iterative improvement of strategies by playing against oneself, but develops conventions between players and does not cooperate with other unseen strategies~\cite{Adam2018Learning,Hu2020OtherPlayFZ}. 
Other-play~\cite{Hu2020OtherPlayFZ} is proposed to break such conventions by adding permutations to one of the strategies.
However, this approach may be reduced to self-play if the game or environment does not have symmetries or has unknown symmetries. 
Another approach is population-based training (PBT)~\cite{PBT,HARL}, which trains strategies by interacting with each other in a population.
However, PBT does not explicitly maintain diversity and therefore does not coordinate with unseen partners\cite{FCP}. 

To achieve the goal of ZSC, recent research has focused on training robust strategies that use diverse populations of strategies~\cite{FCP,TrajDi,MEP}. 
Fictitious co-play (FCP)~\cite{FCP} obtains a population of periodically saved checkpoints during self-play training with different seeds and then trains the best response to the pre-trained population. 
TrajeDi~\cite{TrajDi} also maintains a pre-trained self-play population but encourages distinct behavior among the strategies. 
The maximum entropy population (MEP)~\cite{MEP} method proposes population entropy rewards to enhance diversity during pre-training. It employs prioritized sampling to select challenging-to-collaborate partners to improve generalization to previously unseen policies. 
Furthermore, methods such as MAZE~\cite{Xue2022Heter} and CG-MAS~\cite{Mahajan2022Gen} have been proposed to improve generalization ability through coevolution and combinatorial generalization.
The most recent work of LIPO~\cite{charakorn2023generating} pursues solutions compatible with their partner agents but incompatible with others in the population. 
LIPO is also based on the framework of the PBT algorithm, whose objective function consists of two fixed teams: maximizing the rewards when it pairs with itself, but minimizing the rewards when it pairs with other agents in the population.
Our previous research~\cite{li2023cooperative} emphasized the collaboration with unfamiliar AI agents through the continuous formulation and optimization of objectives. In this study, we expand the scope of \framework to facilitate cooperation with actual human players, which is validated by human players on the human-AI experimental platform we have developed.

\paragraph{Open-Ended Learning.}
Another related area of research is open-ended learning, which aims to continually discover and approach objectives~\cite{Srivastava2012Comtinually,Team2021OpenEndedLL,Meier2022Open}.
In MARL, most open-ended learning methods focus on zero-sum games, primarily posing adaptive objectives to expand the frontiers of strategies~\cite{NIPS2017_3323fe11,psrorn,pipelinepsro,yaodong_diverse,YingWen2021openenned,mcaleer2022self}.
In the specific context of ZSC, the MAZE method~\cite{Xue2022Heter} uses open-ended learning by maintaining two populations of strategies and partners and training them collaboratively throughout multiple generations. 
In each generation, MAZE pairs strategies and partners from the two populations and updates them together by optimizing a weighted sum of rewards and diversity. 
This method co-evolves the two populations of strategies and partners based on naive evaluations such as best or worst performance with strategies in partners.
Our proposed method, \framework, combines GFG and P-GFG in open-ended learning to evaluate and identify the cooperative ability of strategies to solve cooperative incompatibility efficiently with theoretical guarantee.

\paragraph{Coordination Graphs.} Another domain closely associated with our research is the field of coordination graphs (CGs)~\cite{guestrin2002coordinated}. CGs offer an influential framework for modeling intricate interactions and dependencies among multiple agents, facilitating the representation and examination of cooperative problem-solving scenarios. Typically, CGs comprise a vertex (agent) set and an edge set, with each state potentially possessing its distinct coordination graph\cite{guestrin2002coordinated}.
Deep CGs~\cite{bohmer2020deep} refine the joint value function of all agents per coordination graph, factoring payoffs between pairs of agents to strike a flexible balance between representational capacity and generalization. In ad hoc teamwork, Deep CGs have been employed by Rahman et al. (2021), albeit under full observability. The Generalized Policy Learning (GPL) approach~\cite{rahman2021towards} leverages Deep CGs in ad hoc teamwork to learn agent models and joint-action value models in the face of varying team compositions.
Moreover, the Partially Observable GPL (PO-GPL) method~\cite{rahman2022general} broadens the application of GPL by extending it from the fully observable problem domain to the partially observable one. 
In contrast to CGs and subsequent approaches, our graph-based games conceptualize the cooperative task from the standpoint of the normal-form game rather than the stochastic game. 

\section{Preliminaries}
\label{sec:preli}
\paragraph{Normal-Form Game.}
A two-player normal-form game is defined as a tuple $(N, \gA, \rvw)$, where $N=\{1,2\}$ is a set of two players, indexed by $i$, $\gA=\gA_1 \times \gA_2$ is the joint action space, and $\rvw=(w_1, w_2)$ with $w_i: \gA \rightarrow \mathbb{R}$ is a reward function for the player $i$. 
In a two-player common payoff game, two-player rewards are the same, meaning $w_1(a_1,a_2)=w_2(a_1,a_2)$ for $a_1,a_2 \in \gA$.

\paragraph{\changes{Open-Ended Meta-Games and Analysis.}}
\changes{
A meta-game is characterized by the tuple $(\gN, \gS, \gM)$, where $\gN$ represents the set of players, $\gS$ is a set of policies (like a set of RL models), and $\gM$ designates the meta-game payoff matrix. 
In common payoff cooperative meta-game with two players, the meta-game payoff is 
$
    \gM =\left\{\phi(S_1, S_2): (S_1, S_2)\in \gS^1\times \gS^2\right\},
$
where $\phi: \gS^1\times \gS^2 \rightarrow \mathbb{R}$ is the utility function. The meta-game payoff is a symmetric matrix.
The primary distinction between a meta-game and a normal-form game lies in their respective actions. In the case of a meta-game, a strategy such as a DRL model constitutes an action, as opposed to the atomic actions (e.g., up and down) typical of normal-form games.
Open-ended meta-game refers to the continual learning process, wherein new strategies are persistently generated and incorporated into the strategy sets throughout the training phase.}

\changes{
Empirical Game-Theoretic Analysis (EGTA) is the study of finding meta-strategies based on experience with prior strategies~\cite{Walsh2002Analyzing,Karl2018Generalised}. 
An empirical game is built by discovering strategies and meta-reasoning about exploring the strategy space~\cite{NIPS2017_3323fe11}.
Furthermore, empirical gamescapes (EGS) are defined as the convex hull of the payoff vectors of all strategies~\cite{psrorn}.
Given a population $\gN$ of $n$ strategies, the empirical gamescapes is often defined as 
\begin{equation}
\begin{aligned}
    \gG &= \{\text{convex mixture of rows of}~\gM \} \\
    &=\left\{
    \sum_i \alpha_i \cdot \vm_i: \bf{\alpha} \geq 1, \bf{\alpha}^T\bf{1}=1, \vm_i=\gM_{[i,:]}
    \right\}
\end{aligned}
\end{equation}
where $\gM$ is the empirical payoff matrix with the expected outcomes for each joint strategy.}

\paragraph{Cooperative Theoretic Concepts.} In the present study, our main focus is on characteristic function games, a particular class of cooperative games in which a coalition's generated value is represented by a characteristic function~\cite{chalkiadakis2011computational}. We consider a set of players $\gN = \{1, \dots, n\}$, where a coalition is denoted as a subset of players $\gN$, symbolized by $C \subseteq \gN$. The player set $\gN$ is also known as the grand coalition. A characteristic function game, denoted by $G$, consists of a pair $(N, v)$, where $N = \{1, \dots, n\}$ represents a finite, non-empty set of agents, and $v: 2^N \rightarrow \mathbb{R}$ is the characteristic function. This function assigns a real number $v(C)$ to each coalition $C \subseteq N$, with the number $v(C)$ typically considered as the coalition's value.

In addition, in this study, we examine transferable utility games (TU games), which are based on the underlying assumption that the coalitional value $v(C)$ can be divided among the members of $C$ in any way agreed upon by the members of the coalition. This assumption allows for greater flexibility in analyzing cooperative problem solving and the distribution of the resulting payoffs.

Shapley Value~\cite{shapley1971} is one of the important solution concepts for characteristic function games \cite{chalkiadakis2011computational,Bezalel2007intro}.
The Shapley Value aims to distribute fairly the collective value, like the rewards and cost of the team across individuals by each player's contribution.
Taking into account a coalition game $(\gN,v)$ with a strategy set $\gN$ and characteristic function $v$, the Shapley Value of a player $i\in \gN$ could be obtained by 

\begin{equation}
SV(i)=\frac{1}{n!}\sum_{\pi\in\Pi_\mathcal{N}} v(P_i^\pi \cup \{i\}) - v(P_i^\pi),
\label{eq:SV}
\end{equation}
where $\pi$ is one of the one-to-one permutation mappings from $\gN$ to itself in the permutation set $\Pi$ and $\pi(i)$ is the position of player $i \in \gN$ in permutation $\pi$. $P_i^\pi=\{j\in \gN | \pi(j)<\pi(i)\}$ is the set of all predecessors of $i$ in $\pi$.

\paragraph{Graph Theoretic Concepts.} 
In graph theory, a weighted directed graph or network can be represented as $D=(V,E,w)$, consisting of a set of vertices $V$, a set of directed edges $E$, and a weight function $w:E\rightarrow \mathbb{R}^+$ that assigns positive real numbers to each edge. If the graph has the property that $(v,u)\in E$ entails $(u,v)\in E$ and $w(v,u)=w(u,v)$ for all $(u,v)\in E$, it can be considered an undirected weighted graph, also known as a weight graph. An unweighted graph is a special case of a weighted graph, where the weight function assigns a value of $1$ to every edge $(u,v)\in E$.

Node Centrality is a graph theory concept that quantifies a node's relative importance or influence within a network. It is a measure of how central a node is to the overall network structure, with higher centrality values indicating greater importance. 
Degree Centrality is one of the simplest centrality concepts, based on the number of edges connected to a node~\cite{FREEMAN1978215}. 
In directed graphs, it can be further divided into in-degree (number of incoming edges) and out-degree (number of outgoing edges) Centrality.

PageRank~\cite{page1999pagerank} is a centrality measure used to rank web pages in search engine results by estimating the relative importance of each page in a hyperlink network. PageRank(WPG)~\cite{Xing2004Weighted} is an extension of the original PageRank algorithm that considers edge weights in addition to the basic structure of the network. This modification allows the algorithm to handle networks where the strength of connections between nodes varies, such as in citation networks or social networks where the influence of nodes may differ.
The formula of WPG is given as follows:
\begin{equation}
    \sigma(u)=(1-d) + d \sum_{v\in B(u)} \sigma(v) \frac{I_u}{\sum_{p\in R(v)} I_p} \frac{O_u}{\sum_{p\in R(v)}O_p}, 
    \label{eq:wpg}
\end{equation}   
where $d$ is the damping factor set to $0.85$, $B(u)$ is the set of nodes that point to $u$, $R(v)$ denotes the nodes to which $v$ is linked, and $I, O$ are the degrees of inward and outward of the node, respectively.

\section{Cooperative Open-Ended Learning}
\label{sec:cole}

\begin{figure*}[ht!]
\includegraphics[width=\linewidth]{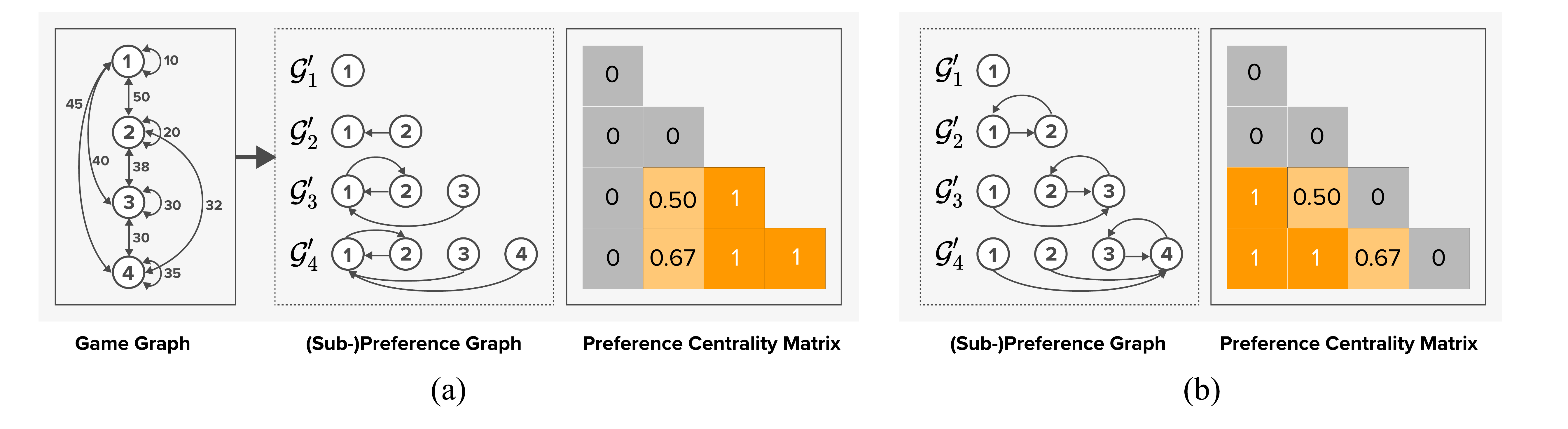}
\centering
\caption{ {The Game Graph, (sub-) preference graph and corresponding preference centrality matrix.}
The (sub-) preference graphs are for all four iterations in the training process, and the corresponding preference in-degree centrality matrix is based on them.
 As can be observed in $\gG^\prime_3$ and $\gG^\prime_4$ in \textit{(a)}, the newly updated strategies fail to be preferred by others and have centrality values of 1, despite an increase in the mean of rewards with all others. 
 In \textit{(b)}, we illustrate an ideal learning process in which a newly generated strategy can achieve higher outcomes than all previous strategies.
}
\label{fig:game_graph}
\end{figure*}

\subsection{Graphic-Form Games (GFGs)} 
It is important to evaluate cooperative incompatibility and identify those failed-to-collaborate strategies to conquer cooperative incompatibility.
Therefore, we propose graphic-form games (GFGs) to reformulate normal-form cooperative games from the perspective of game theory and graph theory, which is the natural development of empirical games~\cite{psrorn}.
The definition of GFG is given below.

\begin{definition}[Graphic-Form Game]
    Let a collection of strategies $\gN = \{1, 2, \cdots, n\}$ be given, where each strategy could be a parameterized network or a human player. A two-player graphic-form game (GFG) can be defined as a triplet $\gG = (\gN, \rmE, \rvw)$, which can be represented as a weighted directed graph. Here, $\gN$, $\rmE$, and $\rvw$ represent the sets of nodes, edges, and weights, respectively. For a given edge $(i, j)$, $\rvw(i, j)$ denotes the anticipated outcome when strategy $i$ competes against strategy $j$. The visual representation of a GFG is referred to as a game graph.
\end{definition}
The payoff matrix of $\gG$ is denoted as $\gM$, where $\gM(i,j)=\rvw(i,j), \forall i,j \in \gN$.
Our goal is to improve the upper bound of other strategies' outcomes in the cooperation within the population, which implies that the strategy should be preferred over other strategies.

Moreover, we propose preference graphic-form games (P-GFGs) as an efficient tool to analyze the current learning state, which can profile the degree of preference for each node in GFGs.
Specifically, P-GFG is a subgraph of GFG, where each node only retains the out-edge with maximum weight among all out-edges except for its self-loop.
Given a GFG $(\gN, \rmE, \rvw)$, the P-GFG could be defined as $\gG^\prime = \{\gN,\rmE^\prime, \rvw\}$, where $\rmE^\prime=\{(i,j) | \argmax_j \rvw(i, j), \forall j\in \{\gN\backslash i\}, \forall i \in \gN\}$ is the set of edges. The graphic representation of P-GFG is called a preference graph. 
\changes{
The game graph delineates the interaction of players within a weighted directed graph, in which the weight signifies the payoff or utility derived from two participating agents. 
However, the preference graph pertains to the depiction of each agent identifying the partner with whom they can achieve the most significant utility in the population. 
This is the underlying principle behind the term ``preference'': an agent demonstrates a strong tendency to cooperate with a particular endpoint agent to maximize their utility. }

To deeply investigate the learning process, we further introduce the \textit{sub-preference graphs} based on P-GFGs, which aim to reformulate previous learning states and analyze the learning behavior of the algorithm.
Suppose that there is a set of sequentially generated strategies $\gN_n=\{1,2,\cdots,n\}$, where the index also represents the number of iterations for simplicity.
For each previous iteration $i<n$, the sub-preference game form graph is denoted as $\{\gN_i, \rmE^\prime_i,\mb{w}_i\}$, where $\gN_i=\{1,2,\cdots,i\}$ is the set of strategies in iteration $i$, and $\rmE^\prime_i, and\ \mb{w}_i$ are the corresponding edges and weights.

The semantics of the preference graph is that a strategy or node $i$ prefers to play with the tailed node to achieve the highest results.
In other words, the more in-edges one node has, the more cooperative ability this node can achieve.
Ideally, if one strategy can adapt well to all others, all the other strategies in the preference graph will point to this strategy.
To evaluate the adaptive ability of each node, the centrality concept is introduced into the preference graph to evaluate how a node is preferred.
\begin{definition}[Preference Centrality]
    Given a P-GFG $\{\gN,E^\prime, \rvw\}$, preference centrality of $i\in \gN$ is defined as,
    $$
    \eta(i)=1- \operatorname{norm}(d_i),
    $$
    where $d_i$ is a graph centrality metric to evaluate how the node is preferred, and $\operatorname{norm}:=\mathbb{R}\rightarrow [0,1]$ is a normalization function.
\end{definition}

Note that the $d$ is a kind of centrality that could evaluate how much a node is preferred.
A typical example of $d$ is the centrality of degrees, which calculates how many edges point to the node. 
\changes{In this work, our primary choice for the graph centrality metric is in-degree centrality.}

\begin{figure}[t!] 
\centering 
\subfigure[\changes{The illustration showcases an instance of an FCP agent's collaboration failure with specified human participants. This figure is referenced from HSP~\cite{HSP}.}] {   
    \includegraphics[scale=0.6]{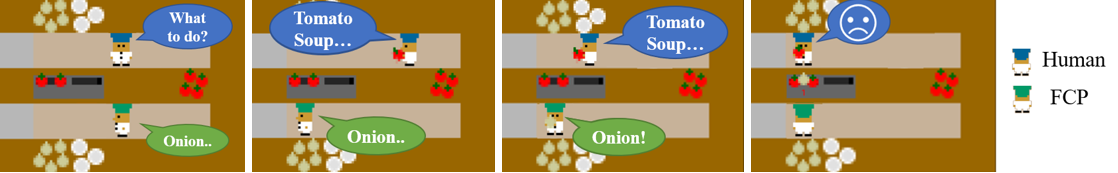}  
\label{fig:hsp_eg}
}  
\subfigure[\changes{Cooperative incompatibility issue in MEP training process.}] {   
    \includegraphics[scale=0.34]{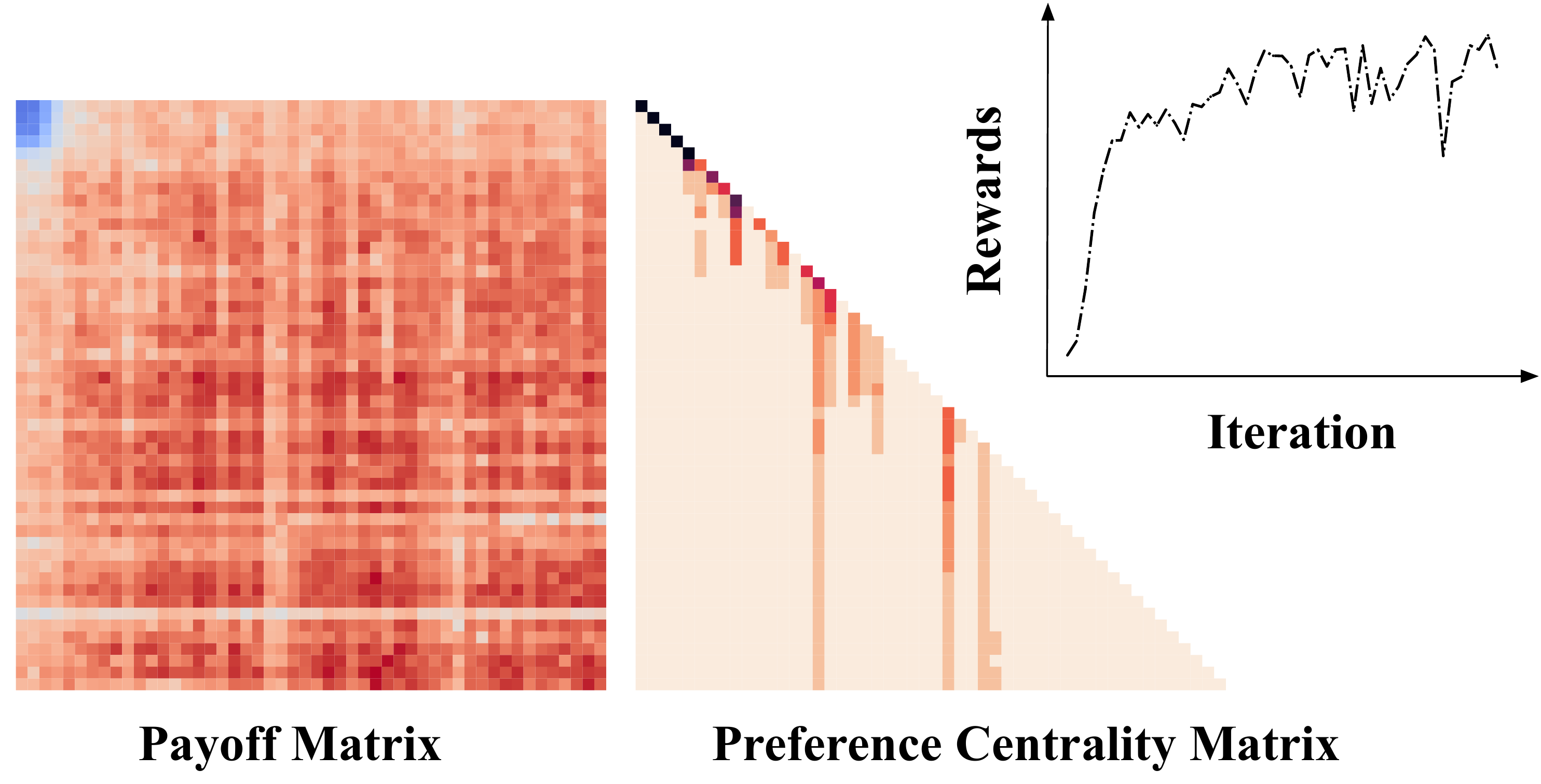}  
\label{fig:mep_eta}
}   
\caption{\changes{Motivating examples. Fig.~\ref{fig:hsp_eg} features an analysis conducted by HSP \cite{HSP}. The FCP agent converges to a fixed pattern of exclusively preparing onion soup, thereby failing to establish coordination with a human participant who prefers making tomato soup.
Fig.~\ref{fig:mep_eta} shows a training process of the MEP algorithm with several comparative incompatibility problems.
The payoff matrix of each strategy during training and the corresponding preference centrality matrix of the MEP algorithm in the Overcooked. 
A deeper shade of red in the payoff matrix signifies higher utility. 
The darker the color in the preference centrality matrix, the lower the centrality value, and the more other strategies prefer it.}
}
\label{fig:motivating_examples}     
\end{figure}

\changes{
Fig.~\ref{fig:game_graph} provides examples of the learning processes of two algorithms, one with and one without a cooperative incompatibility issue.
In Fig.~\ref{fig:game_graph}(a), illustrative of the algorithm possessing this cooperative incompatibility, the updated strategy is observed to fall short in elevating the cooperative utility achieved with other strategies after the second generation. There are no edges pointing towards the node which validates this further, and is confirmed by the preference centrality matrix, where the strategy at hand is marked with a $\eta=1$. This suggests that no nodes wish to collaborate with the updated strategies.
On the other hand, Fig.~\ref{fig:game_graph}(b), devoid of cooperative incompatibility, portrays a different scenario. In this depiction, every other strategy within the preference graph directs towards the most recent strategy, signifying the continuous enhancement of the cooperative capability.}

\subsection{\changes{Motivating Examples}}
\changes{
The challenge of training AI agents to effectively coordinate with unfamiliar AI counterparts or human players constitutes a substantial impediment. 
Predominantly trained through reward maximization methods, AI agents typically demonstrate deterministic strategies which often diverge from those employed by human players or other AI agents, thereby resulting in a coordination failure with previously unseen players~\cite{Deheng2020Towards,Yiming2023Towards,HSP}.
Fig.~\ref{fig:hsp_eg}, proposed by HSP~\cite{HSP}, illustrates a motivating example in the Overcooked game of how an agent player could collaborate with some players and fail to collaborate with other players.
In this example of the Overcooked game, chefs could accomplish two recipes: one calls for three onions and the other demands three tomatoes. The FCP algorithm, however, tends to settle into a distinctive pattern of cooking solely onion soup~\cite{HSP}. Consequently, this FCP agent can achieve high scores when paired with similarly styled human players. Nevertheless, its performance declines significantly when paired up with a player who prefers the tomato soup style, resulting in a disruption of the human player's tomato soup cooking plan, thereby failing to complete the required recipe.
Further instances of miscoordination can be observed in real-world scenarios, such as autonomous driving, where the question may arise whether the policy should be aggressive or conservative.
}

\changes{
Nevertheless, miscoordination issues exist in the training of some population-based ZSC algorithms.
An analysis of the MEP algorithm~\cite{MEP}, as depicted in Fig.~\ref{fig:mep_eta}, reveals a cooperative incompatibility evident during the learning process in the Overcooked environment~\cite{HARL}.
The figure on the left represents the payoff matrix, where a deeper shade of red signifies higher utility. In the accompanying preference indegree centrality matrix, a darker color is indicative of a strategy being more preferred by others. Throughout the MEP learning process, despite a consistent improvement in mean rewards (as depicted in the upper-right of Fig.~\ref{fig:mep_eta}), significant cooperative incompatibility problems arise post a certain training duration. 
At this juncture, many strategies show an inclination to engage with earlier strategies, represented by darker shades, over newer strategies, in an effort to secure higher rewards.
Therefore, addressing this collaboration incompatibility is crucial to enhance the adaptability of AI agents in coordinating with unseen partners.
}

\subsection{\changes{Cooperative Incompatibility}}
\changes{Fig.~\ref{fig:game_graph} provides examples of cooperative incompatibility, specifically concerning human-AI coordination and ZSC algorithm learning. 
Consequently, this section will delve deeper into a conceptual understanding of cooperative incompatibility and formally delineate this issue within the learning process, utilizing the tools previously introduced such as P-GFGs and preference centrality.}

\begin{definition}[\changes{Cooperative Incompatibility}] \changes{
Cooperative incompatibility pertains to the occurrence wherein a participant, either an AI agent or a human player, is unable to align with certain specific partners, who may equally be AI agents or human players.}
\end{definition}

\paragraph{\changes{Cooperative Incompatibility in the learning of ZSC algorithms.}} \changes{A cooperative incompatibility issue arises in a population-based learning algorithm when the new strategy $s_t$, produced by the algorithm at step $t$, is unable to boost cooperative utility in conjunction with existing strategies in the population. 
This is to say, the preference centrality $\eta(t)$ of $s_t$ is greater than 0. 
Ideally, an algorithm devoid of cooperative incompatibility issues will maintain a preference centrality of 0 at each step. This implies that the updated strategy enhances the collaborative utility of strategies within the population, prompting them to prefer cooperation with the updated strategy.
}

\changes{
In this study, our aim is to resolve the cooperative incompatibility in the learning process of population-based algorithms, ultimately paving the way to tackle cooperative incompatibility issues in AI-AI and AI-human coordination.
}

\begin{figure*}[t]
    \centering
\includegraphics[width=\linewidth]{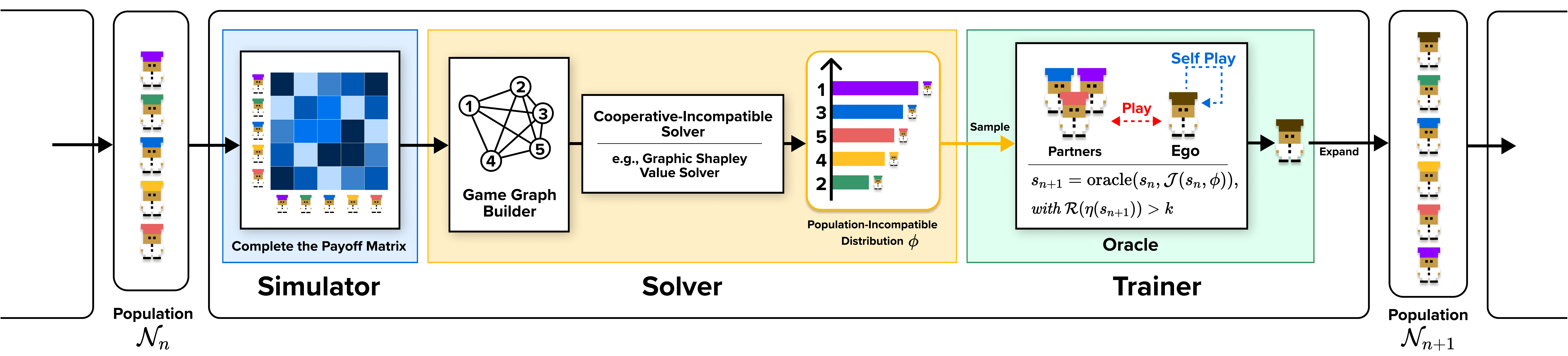}
    \caption{
    An overview of one generation in \framework: The solver derives the cooperative incompatible distribution $\phi$ using a cooperative incompatibility solver, which can be any algorithm that evaluates cooperative contribution. The trainer then approximates the relaxed best response by optimizing individual and cooperative compatible objectives. The oracle's training data is generated using partners selected based on the cooperative incompatibility distribution and the agent's strategy. Finally, the approximated strategy $s_{n+1}$ is added to the population, and the next generation begins.
    }
    \label{fig:cole}
\end{figure*}

\subsection{Cooperative Open-Ended Learning Framework}

\changes{To address the problem of cooperative incompatibility, we refine our comprehension of the objective by formulating empirical gamescapes~\cite{psrorn} from zero-sum games to our studied shared payoff GFGs. 
This provides us with a geometric representation of player strategies within a GFG $\{\gN, \rmE, \rvw\}$, thereby capturing the diversity and adaptive capacity of cooperative strategic behaviour.
However, it is not efficient to directly learn how to extend the EGS in common payoff games to cooperate effectively with unseen partners. }

To conquer cooperative incompatibility, the natural idea is to learn with the mixture of cooperative incompatible strategies on the most recent population $\gN$ to improve gamescape.
Given a population $\gN$, we present \textit{cooperative incompatible solver} to assess how strategies collaborate, especially with those strategies that are difficult to collaborate with.
The solver derives the cooperative incompatible distribution $\phi$, where strategies that do not coordinate with others have higher probabilities.
We also optimize the cooperative incompatible mixture over the individual objective, which is the cumulative self-play rewards to improve the adaptive ability with expert partners.
To simplify, we name it the individual and cooperative incompatible mixture (IPI mixture).
We use an approximate oracle to approach the best response over the IPI mixture.
Given strategy $s_n$, the oracle returns a new strategy $s_{n+1}$ :
$
    s_{n+1} = \operatorname{oracle}(s_{n+1}, \gJ(s_{n}, \phi)),
$
with $\eta(s_{n+1})=0$ , if possible. 
$\gJ$ is the objective function as follows,
\begin{equation}
    \label{eq:obj}
    \gJ(s_n,\phi) = \mathbb{E}_{p\sim \phi}\rvw(s_n,p) + \alpha \rvw(s_n,s_n).
\end{equation}
\changes{Here, $\alpha$ represents the balancing hyperparameter, while $\phi$ denotes the cooperative incompatibility distribution calculated by the solver , in which strategies that fail to coordinate with others are assigned higher probabilities.}
The objective consists of the cooperative compatible objective and the individual objective.
The cooperative compatible objective aims to train the best response to those failed-to-collaborate strategies, and the individual objective aims to improve the adaptive ability with expert partners.
We call the best response the best-preferred strategy if $\eta(s_{n+1})=0$. 

However, arriving at the best-preferred strategy with $\eta(s_{n+1})=0$ is hard or even impossible.
Therefore, we seek to approximate the best-preferred strategies by relaxing the best strategy to the strategy whose preference centrality ranks top $k$.
The approximate oracle could be rewritten as 
\begin{equation}
    s_{n+1} = \operatorname{oracle}(s_n, \gJ(s_n, \phi_n)),
\ with\ 
\mathcal{R}(\eta(s_{n+1}))>k,
\label{eq:oracle_approx}
\end{equation}
\changes{where $\mathcal{R}(\cdot)$ serves as the ranking function, wherein a lower preference centrality value corresponds to higher ranks. $\phi_n$ refers to the distribution at generation $n$. For simplicity, we may forego the use of the subscript $n$ in the remainder of the paper.
As illustrated by the equation, the terminal condition for a single generation of oracle training has been moderated to satisfy the criterion wherein the rank of preference centrality of optimized strategy resides within the top $k$. Instead of calculating best-preferred strategy, the approximate oracle aims to output best-k-preferred strategy to the IPI mixture.}

We extend the approximated oracle to open-ended learning and propose \framework (Fig.~\ref{fig:cole}).
The \framework iteratively updates new strategies that approximate the best-preferred strategies to the cooperative incompatible mixture and the individual objective.
The simulator completes the payoff matrix with the newly generated strategy and others in the population.
The solver aims at derive the cooperative incompatible distribution of the Game Graph builder and the cooperative incompatible solver.
The trainer uses the oracle to approximate the best-preferred strategy to the cooperative incompatible mixture and individual objective and outputs a newly generated strategy which is added to the population for the next generation.

Although we relax the best-preferred strategy to the strategy in the top $k$ centrality in the constraint, \framework still converges to a local best-preferred strategy with zero preference centrality.
Formally, the convergence theorem of the local best preferred strategy is given as follows.
 \begin{restatable}{theorem}{FirstTHM}
    \label{thm: converge}
\changes{Let $s_0\in \gS$ be the initial strategy and $s_i=\operatorname{oracle}(s_{i-1})$ for $i \in \mathbb{N}$.
Under the effective functioning of the approximated oracle as characterized by Eq.~\ref{eq:oracle_approx}, we can say that the sequence $\{s_i\}$ for ${i\in \mathbb{N}}$ could converge to a local optimal strategy $s^*$, i.e., the local best-preferred strategy.}
 \end{restatable}
 \begin{proof} 
 See Appendix~\ref{appendix:proofs_thm}.
\end{proof}
Furthermore, if we choose in-degree centrality as the preference centrality function, the convergence rate of \framework is Q-sublinear.
\begin{restatable}{corollary}{FirstLEMMA}
    \label{lemma: converge_rate}
Let $\eta: \gG^\prime \rightarrow \mathbb{R}^n$ be a function that maps a P-GFG to its in-degree centrality, the convergence rate of the sequence $\{s_i\}$ is Q-sublinear concerning $\eta$.
 \end{restatable}
\begin{proof}
See Appendix~\ref{appendix:proofs_corollary}.
\end{proof}

\section{Practical Algorithm}
\label{sec:cole_sv}
\changes{In an endeavour to tackle cooperative incompatibility issues in common-payoff games with two players, we have developed two practical algorithms, \algoR and \algo. 
These are based on the \framework, and are specifically designed to reconcile cooperative incompatibility and augment zero-shot coordination capabilities. 
As depicted in Fig.~\ref{fig:cole}, these algorithms, at each generation, accept an input population $\gN$ and from it, derive a local best-preferred strategy that is then appended to $\gN$ in order to extend the population. The process of generating this strategy involves the collaboration of the simulator, solver, and trainer modules.
While \algo and \algoR utilize a common simulator and trainer, they each employ a distinct solver. Computation of the payoff matrix $\gM$ for the input population $\gN$ is carried out by the simulator, with each element $\gM(i,j)$ where $i,j\in \gN$ symbolizing the cumulative rewards of players $i$ and $j$ at both respective starting positions. 
The solver's role involves assessing and identifying strategies that have failed to collaborate effectively, achieved through the calculation of an incompatible cooperative distribution.
A Graphic Shapley Value solver is integrated into \algo enabling it to measure the cooperative efficacy of each strategy relative to all others. This is achieved by implementing the weighted PageRank (WPG)~\cite{Xing2004Weighted} from graph theory into the Shapley Value. This technique allows for the evaluation of adaptability, particularly in relation to those strategies that have failed to collaborate effectively. As a contrast, \algoR incorporates a simplified solver which computes the cooperative ability by determining the payoff with other members of the population. Subsequently, the trainer approximates the strategy that is most preferred against the recently updated population.}

\begin{algorithm}[t!]
\caption{\changes{Practical Algorithms}}
\label{alg:cool}
\begin{algorithmic}[1]
\STATE \algorithmicrequire population $\gN_{0}$, the sample times $a,b$ of $\gJ_i,\gJ_c$, hyperparameters $\alpha,k$\changes{, solver flag $FLAG$}

\FOR{$t = 1,2,\cdots, $}
    \STATE {/* Step 1: Completing the payoff matrix */}
    \STATE $\gM_n \leftarrow \operatorname{Simulator}(\gN_t)$
    \STATE{/* Step 2: Solving the cooperative incompatibility distribution */}
    \changes{
    \IF{FLAG is ``SV''}
    \STATE{/* Selecting Graphic Shapley Value Solver */}
    \STATE $\phi = \operatorname{Graphic\ Shapley\ Value}(\gN_t)$ by Algorithm~\ref{algo:solver}
    \ELSE
    \IF{FLAG is ``R''}
    \STATE{/* Selecting Reward Solver */}
    \STATE $\phi = \operatorname{Reward\ Solver}(\gN_t)$
    \ENDIF
    \ENDIF
    }
    \STATE{/* Step 3: Approximate the best-preferred strategy */}
    \STATE $\gJ = \sum^b_{p\sim \phi}\phi(p)\rvw(s_t,p) + \alpha \sum^a\rvw(s_t,s_t)$, where $s_t=\gN_t(t)$, $\phi$ is updated each time by Eq~\ref{eq:SUCG}
    \STATE {$s_{t+1} = \operatorname{oracle}(s_t, \gJ)$} with $\mathcal{R}(\eta(s_{n+1}))>k$
    \STATE{/* Step 4: Expand the population */}
    \STATE $\gN_{t+1} = \gN_{t} \cup \{s_{t+1}\}$ 
    \ENDFOR
\end{algorithmic}
\end{algorithm}

\subsection{\changes{Solvers}}

\label{sec:solver}

\paragraph{\changes{Graphic Shapley Value Solver for \algo.}}
The graphic Shapley value solver is proposed to calculate the cooperative incompatible distribution as a mixture to approximate the best-preferred strategies in the recent population and overcome cooperative incompatibility.
Specifically, we combine the Shapley Value~\cite{shapley1971} solution, an efficient single solution concept for cooperative games to assign the obtained team value between individuals, with our GFG to evaluate and identify the strategies that did not cooperate.
To apply the Shapley Value, we define an additional characteristic function to evaluate the value of the coalition.
Formally, given a coalition $C\subseteq \gN$, we have the following:
  $  v(C) = \mathbb{E}_{i\sim C,j\sim C}\sigma(i)\sigma(j)\rvw(i,j),$
where $\sigma$ is a mapping function that evaluates how badly a node performs on its game graph.
We use the characteristic function to evaluate the coalition value of how it could cooperate with those hard-to-collaborate strategies.

We take the inverse of WPG~\cite{Xing2004Weighted} on the game graph as the metric $\sigma$.
WPG is proposed to assess the popularity of a node in a complex network.
The formula of WPG is given as follows:
\begin{equation}
    \hat\sigma(u)=(1-d) + d \sum_{v\in B(u)} \hat\sigma(v) \frac{I_u}{\sum_{p\in R(v)} I_p} \frac{O_u}{\sum_{p\in R(v)}O_p}, 
    \label{eq:wpg_sigma}
\end{equation}
where $d$ is the damping factor set to $0.85$, $B(u)$ is the set of nodes that point to $u$, $R(v)$ denotes the nodes to which $v$ is linked, and $I, O$ are the degrees of inward and outward of the node, respectively.
Therefore, the metric $\sigma$ evaluates how unpopular a node is and equals the inverse of the WPG value $\hat\sigma$.

Then we calculate the Shapley Value of each node by taking a characteristic function in equation~\ref{eq:SV}, named the graphic Shapley Value.
We use Monte Carlo permutation sampling~\cite{Castro2009PolynomialCO} to approximate the Shapley Value, which can reduce the computational complexity from exponential to linear time.
After inverting the probabilities of the graphic Shapley Value, we get the cooperative incompatible distribution $\phi$, where strategies that fail to collaborate with others have higher probabilities.
The details are given in Algorithm~\ref{algo:solver}.

\paragraph{\changes{Reward Solver for \algoR.}} 
\changes{
\algoR substitutes the Shapley value for WPG weighted average rewards within the population.
Specifically, for each player $i\in \gN$, we have
\begin{equation}
    \phi_i = \sum_{j\in \gN} \sigma(i)\sigma(j)\rvw(i,j),
\end{equation}
where $\rvw(i,j)$ is the weight in the game graph, i.e., average payoffs of player $i$ and $j$. $\sigma$ is calculated as same as the calculation of graphic Shapley value solver for \algo. 
The remainder of the reward solver processes are identical to the graphic Shapley value solver.
}

\begin{algorithm}[t!]
\caption{Graphic Shapley Value Solver Algorithm}
\label{algo:solver}
\begin{algorithmic}[1]
\STATE \algorithmicrequire: population $\gN$, the number of Monte Carlo permutation sampling $k$, the size of negative population

\STATE Initialize $\phi = \mb{0}_{|\gN|}$
\FOR{$(1,2,\cdots, k)$}
    \STATE $\pi \longleftarrow \textit{Uniformly sample from } \Pi_\gC$, where $\Pi_\gC$ is permutation set
    \FOR{$i\in \mc{N}$}
    \STATE /* Obtain predecessors of player $i$ in sampled permutation $\pi$ */
    \STATE $S_\pi(i) \longleftarrow \{j\in \mc{N} | \pi(j)<\pi(i)\}$
    \STATE /* Update incompatibility weights */
    \STATE $\phi_i\longleftarrow {\phi}_i + \frac{1}{k}({v(S_{\pi}(i)\cup \{i\})-v(S_{\pi}(i)))}$
    \ENDFOR
    \ENDFOR
\STATE $\phi \longleftarrow \phi/\sum\phi$
\STATE $\phi \longleftarrow (1-\phi)/\sum(1-\phi)$
\STATE \algorithmicensure: $\phi$
\end{algorithmic}
\end{algorithm}

\subsection{Trainer: Approximating Best-Preferred Strategy}

The trainer inputs the cooperative incompatible distribution $\phi$ and samples its teammates to learn to approach the best-preferred strategy against the IPI mixture.

Recall the oracle for $s_n$ : 
$
    s_{n+1} = \operatorname{oracle}(s_{n+1}, \gJ(s_{n}, \phi)),
$
with $\mathcal{R}(\eta(s_{n+1}))>k$.
\algo aims to optimize the best-preferred strategy over the IPI mixture. 
The $\gJ(s_{n}, \phi)$ is the joint objective that consists of individual and cooperative compatible objectives.
The individual objective aims to improve performance within itself and promote adaptive ability with expert partners, formulated as follows:
$
    \gJ_i(s_n) = \rvw(s_n,s_n),
$
where $s_n$ is the ego strategy that needs to be optimized in generation $n$.

And the cooperative compatible objective aims to improve cooperative outcomes with those failed-to-collaborate strategies:
$
    \gJ_{c} = \mathbb{E}_{p\sim \phi}\rvw(s_n,p),
$
where the objective is the expected rewards of $s_n$ with cooperative incompatible distribution-supported partners.
$\rvw$ estimates and records the mean cumulative rewards of multiple trajectories and starting positions.
\changes{The expectation can be approximated as follows:
$\gJ_{c} = \sum^b_1\phi(p^i)\rvw(s_t,p^i),
$
where $b$ is the number of sampling times.}

To balance exploitation and exploration as the learning continues, we present the Sampled Upper Confidence Bound for Game Graph (SUCG) which combines the Upper Confidence Bound (UCB) and GFG to control the sampling for more strategies with higher probabilities or new strategies.
Additionally, we view the SUCG value as the probability of sampling teammates instead of using the maximum item in typical UCB algorithms.
Specifically, in the game graph, we keep the information on the times that a node has been visited.
Therefore, the probability of each node considers both the Shapley Value and visiting times, denoted as $\hat{p}$.
The SUCG for any node $u$ in $\gN$ could be calculated as follows:
\begin{equation}
    \hat\phi(u) = \phi(u) + c\frac{\sqrt{\sum_{i\in \gN} \mb{N}(i)}}{1+\mb{N}(u)},
    \label{eq:SUCG}
\end{equation}
where $c$ is a hyperparameter that controls the degree of exploration and $\mb{N}(i)$ is the visit times of node $i$.
SUCG could efficiently prevent \algo from generating data with a few fixed strategies that did not cooperate, which could lead to loss of adaptive ability.

We conclude \algo as Algorithm~\ref{alg:cool}.
Furthermore, to verify the influence of different ratios of two objectives, we denote \algo with different ratios as 0:4, 1:3, 2:2, and 3:1.
Specifically, \algo with $a:b$ represents different partner sampling ratios for the combining objective, where $a$ is the corresponding times to generate data using self-play for the individual objective, and $b$ is the number of sampling times in $\gJ_c$.
For example, \algo 1:3 trains using self-play once and sampling from the cooperative incompatible distribution as partners three times to generate data and update the objectives.

\section{\changes{Human-AI Experiment Pipeline}}
\label{sec:cole_platform}
\begin{figure}[ht!]
    \centering
    \includegraphics[width=\linewidth]{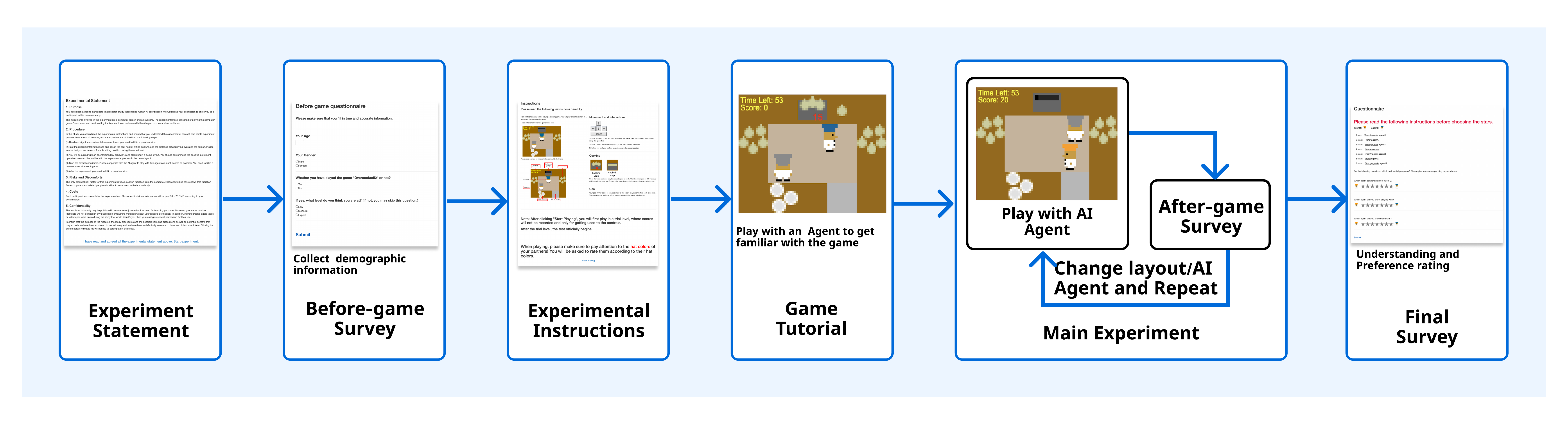}
    \caption{\changes{The illustrated figure characterizes the conceptual architecture of the proposed human-AI experimental pipeline, structured around six critical stages. \textbf{(Stage 1) Experiment Statement} delineates the nature of the experiment, associated risks, and ethical considerations among other relevant aspects. \textbf{(Stage 2) Before-game Survey} principally focuses on the acquisition of participant information. Delving deeper, \textbf{(Stage 3) Experimental Instructions} dispenses extensive procedural guidelines for the experiment. Subsequently, participants are invited to engage in a sequence of trial games to acquaint themselves with the experimental procedure in \textbf{(Stage 4) Game Tutorial}. Proceeding to \textbf{(Stage 5) Main Experiment}, it entails various rounds with a diverse array of differentiated AI agents. Post each round, participants are obliged to fill out a survey. The entire experimental process culminates with a comprehensive evaluation of the collaborative AI counterparts in \textbf{(Stage 6) Final Survey}.
    The pipeline is integrated into one platform designed for seamless interaction. Researchers have the flexibility to tailor the pipeline according to their needs, while participants benefit from a user-friendly interface that enables them to complete the stages with ease.}}
    \label{fig:pipeline}
\end{figure}

\changes{In this section, we present our Overcooked human-AI experiment pipeline, explicitly crafted for a comprehensive and streamlined assessment of the cooperative performance of AI agents, working in collaboration with novice human players. 
Fig.~\ref{fig:pipeline} illustrates the proposed experimental framework integrating human and AI participants, encompassing six crucial sequential stages.
The crux of this pipeline lies in its human-AI evaluation design, integrating a broad array of subjective metrics. These metrics scrutinize individual games involving AI agents, along with providing a comprehensive comparative analysis across the entire participant pool. These encompassing metrics primarily evaluate vital aspects such as intentionality, contribution, and teamwork within human-AI collaborations. Moreover, these metrics facilitate an extension of this evaluation to include parameters such as fluency, preference, and comprehension across all collaborating agents.
Besides, all six steps are integrated into one platform, named COLE-Platform, designed for seamless interaction. Researchers have the flexibility to tailor the pipeline according to their needs, while participants benefit from a user-friendly interface that enables them to complete the stages with ease.
}

\changes{
As shown in Fig.~\ref{fig:pipeline}, 
the initial stage entails the formulation of an experiment statement, where the ethical considerations and potential risks associated with the experiment are thoroughly assessed, alongside the information of pertinent experimental details.
Subsequently, participants are required to complete a preliminary survey, which seeks to gather fundamental information like their familiarity with the game and their age for more nuanced analyses. The next stage comprises detailed experimental instructions and a game tutorial, equipping participants with the necessary knowledge and skills to play the game and allowing them some initial exposure to the game for familiarization purposes.
Then, the process transitions into various game stages, encompassing a game tutorial and the main experiment, which comprises multiple rounds against a range of distinct AI agents. 
After each round, participants are required to complete a survey. The experiment is concluded with a final survey. Comprehensive data, encompassing player trajectories and AI network parameters, are meticulously recorded and stored in the respective database.
In Appendix~\ref{app:screenshot}, screenshots from different phases of the pipeline can be found (Figs.~\ref{fig:platform_state}-\ref{fig:platform_after}).
}
    
\changes{
The remainder of this section is devoted to introducing the Overcooked game, an overview of the Human-AI Experiment Platform, and a discussion on the design of the experimental scale.
}


\begin{figure}[t!] 
\centering 
\subfigure[{Forced Coord.}] {   
    \includegraphics[scale=0.3]{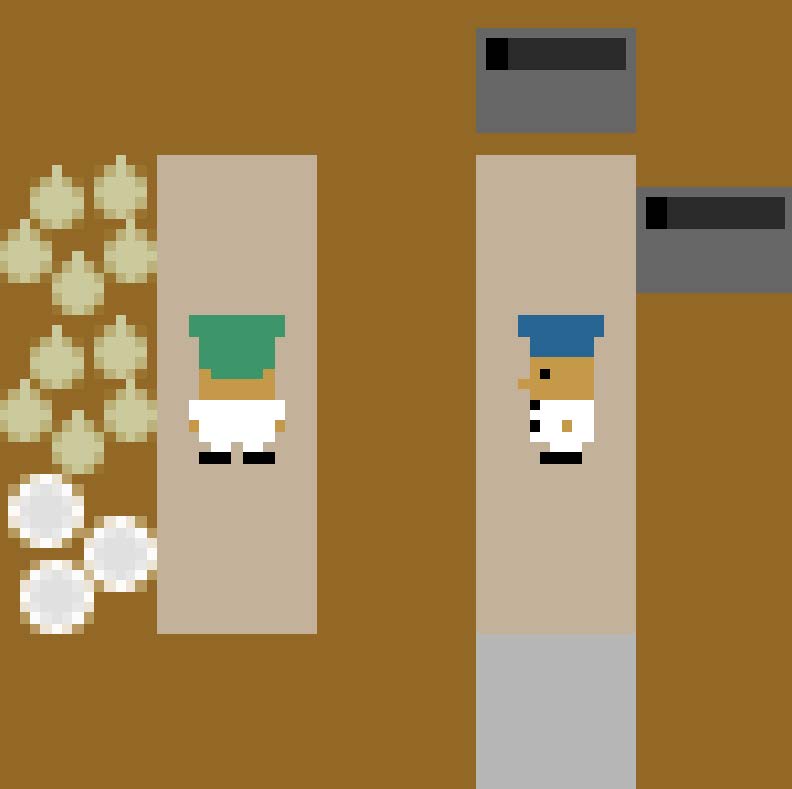}  
}   
\subfigure[{Counter Circ.}] {   
    \includegraphics[scale=0.3]{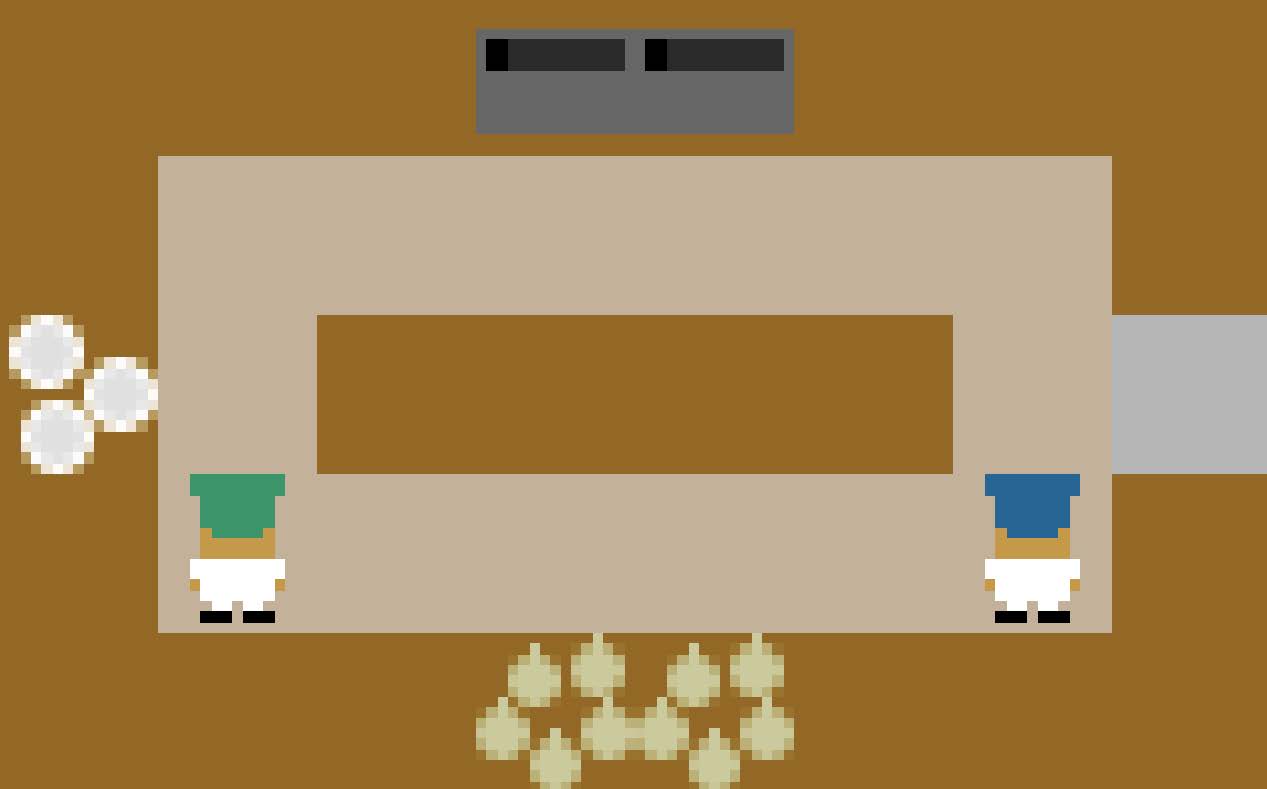}  
} 

\subfigure[{Asymm. Adv.}] {   
    \includegraphics[scale=0.4]{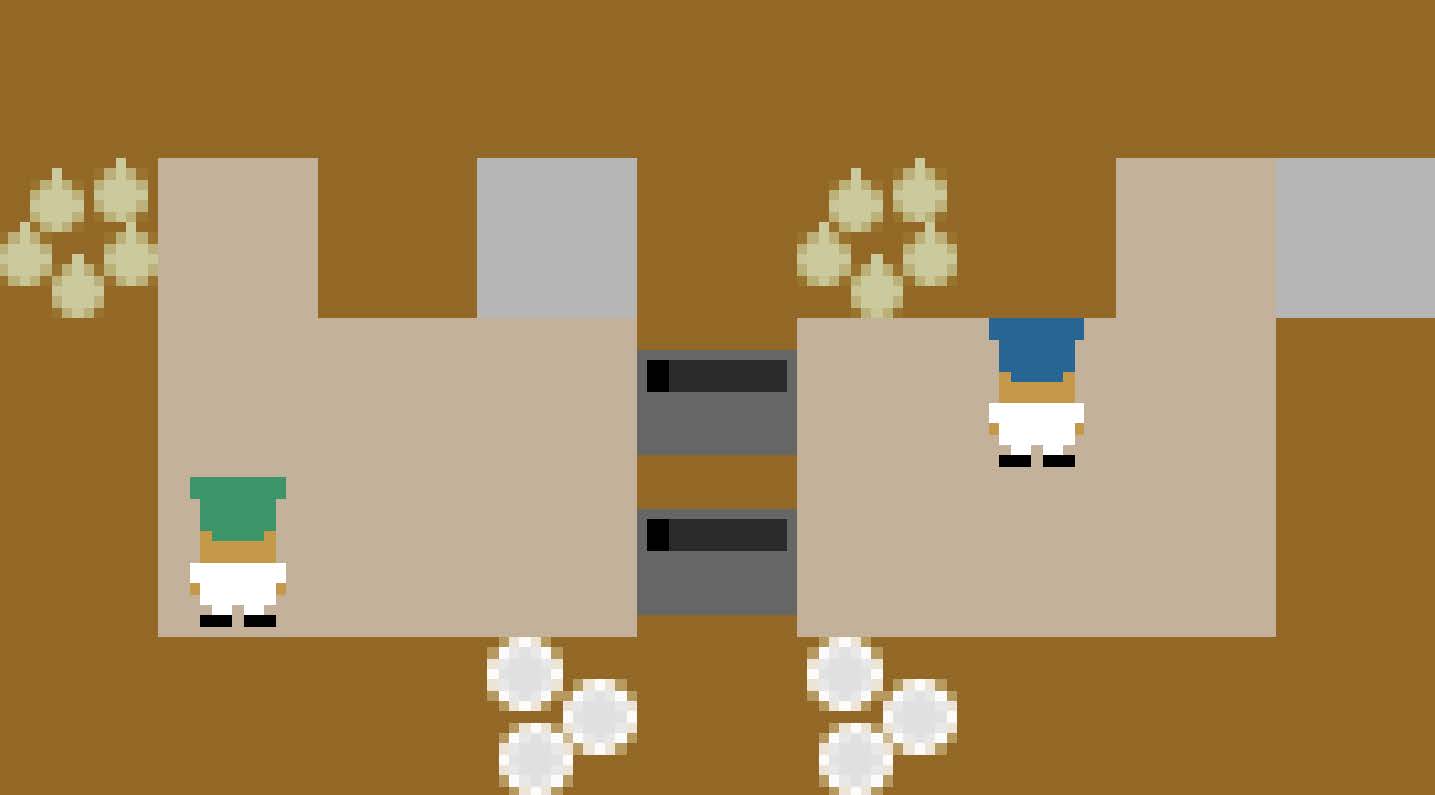}  
}    

\subfigure[{Cramped Rm.}] {\includegraphics[scale=0.34]{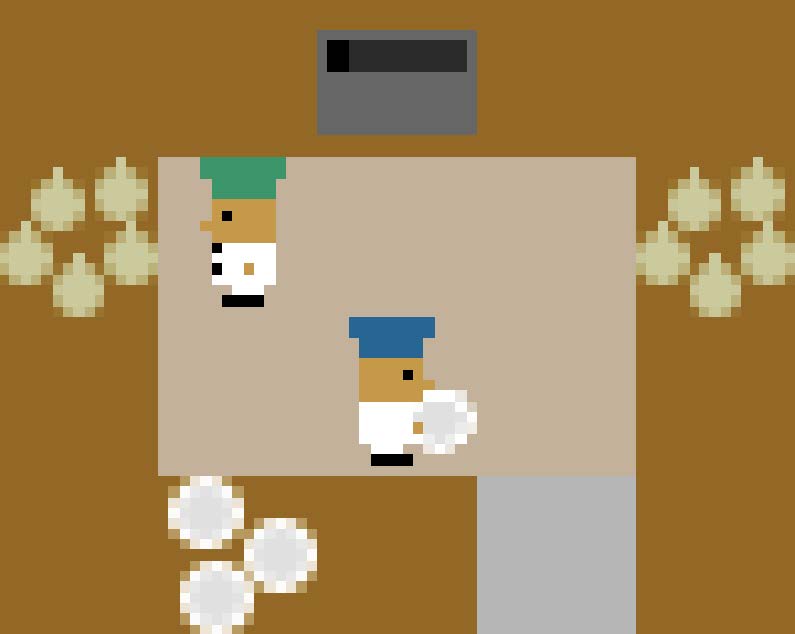}
}   
\subfigure[{Coord. Ring}] {   
    \includegraphics[scale=0.34]{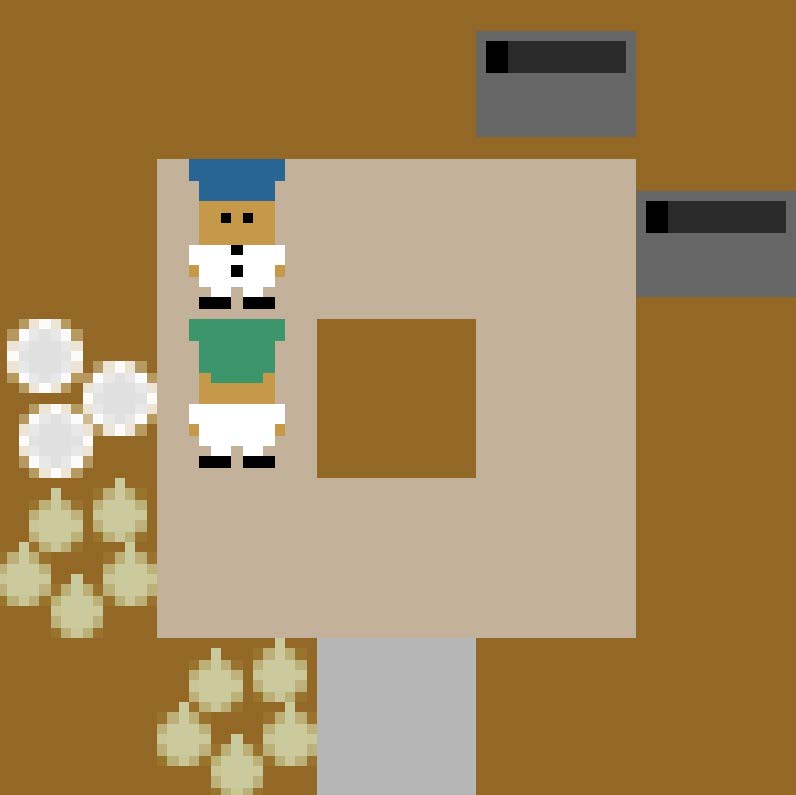}  
}   
\caption{Overcooked environment layouts. The Cramped Rm., Asymm. Adv., and Coord. Ring layouts are more conducive to higher rewards when players cooperate with different partners. On the other hand, the Forced Coord., Counter Circ., and Asymm. Adv. layouts offer distinct designs that serve as ideal testbeds to explore and foster cooperation between players.}     
\label{fig:overcooked-layouts}     
\end{figure}

\subsection{Overcooked Environment}
Our paper implements the platform in the Overcooked environment~\cite{HARL,charakorn2020investigating,knott2021evaluating}, a simulation environment for reinforcement learning derived from the Overcooked!2 video game~\cite{HARL}. The Overcooked environment features a two-player collaborative game structure with shared rewards, where each player assumes the role of a chef in a kitchen, working together to prepare and serve soup for a team reward of 20 points. The environment consists of five distinct layouts: Cramped Room (Cramped Rm.), Asymmetric Advantages (Asymm. Adv.), Coordination Ring (Coord. Ring), Forced Coordination (Forced Coord.), and Counter Circuit (Counter Circ.). Visual representations of these layouts can be found in Fig.~\ref{fig:overcooked-layouts}. The detailed introduction of five layouts is as follows.

\paragraph{Forced Coordination.} The Forced Coordination environment is designed to necessitate cooperation between the two players, as they are situated in separate, non-overlapping sections of the kitchen. Furthermore, the available equipment is distributed between these two areas, with ingredients and plates located in the left section and pots and the serving area in the right section. Consequently, the players must work together and coordinate their actions to complete a recipe and earn rewards successfully.

\paragraph{Counter Circuit.} 
The Counter Circuit layout features a ring-shaped kitchen with a central, elongated table and a circular path between the table and the operational area. In this configuration, pots, onions, plates, and serving spots are positioned in four distinct directions within the operational area. Although the layout does not explicitly require cooperation, players may find themselves obstructed by narrow aisles, prompting the need for coordination to maximize rewards. One example of an advanced technique players can learn is to place onions in the middle area for quick and efficient passing, thereby enhancing overall performance.

\paragraph{Asymmetric Advantages.} In the Asymmetric Advantages layout, players are divided into two separate areas, but each player can independently complete the cooking process in their respective areas without cooperation. However, the asymmetrical arrangement of the left and right sides encourages collaboration to achieve higher rewards. Specifically, two pots are placed in the central area, accessible to both players. The areas for serving and ingredients, however, are completely distinct. The serving pot is placed near the middle on the left side and far from the middle on the right side, with the ingredients area arranged oppositely. Players can minimize their walking time and improve overall efficiency by learning how to collaborate effectively.

\paragraph{Cramped Room.} The Cramped Room layout presents a simplistic environment in which two players share an open room with a single pot at the top and a serving area at the bottom right. In this setup, players could score high even without extensive coordination.

\paragraph{Coordination Ring.} The Coordination Ring layout is another ring-shaped kitchen, similar to the Counter Circuit. However, this layout is considerably smaller than Counter Circuit, with a close arrangement that makes it easier for players to complete soups. The ingredients, serving area, and plates are all in the bottom left corner, while the two pots are in the top right. As a result, this layout allows more easily achieving high rewards.

In summary, the Cramped Rm., Asymm. Adv., and {Coord. Ring} layouts are more conducive to achieving higher rewards when players cooperate with different partners. On the other hand, the Forced Coord., Counter Circ., and Asymm. Adv. layouts, due to their unique designs, provide more suitable testbeds for investigating and promoting cooperation among players.

\subsection{Details on Human-AI Experiment Platform}
We delve into the details of our Human-AI Experiment Platform, offering a comprehensive overview of its components and functionality. To better understand the system's user interface, layout, and functionality, a visual representation of the Human-AI Experiment Platform can be found in Appendix~\ref{app:screenshot}.

\begin{figure}[ht!]
    \centering
    \includegraphics[width=0.9\linewidth]{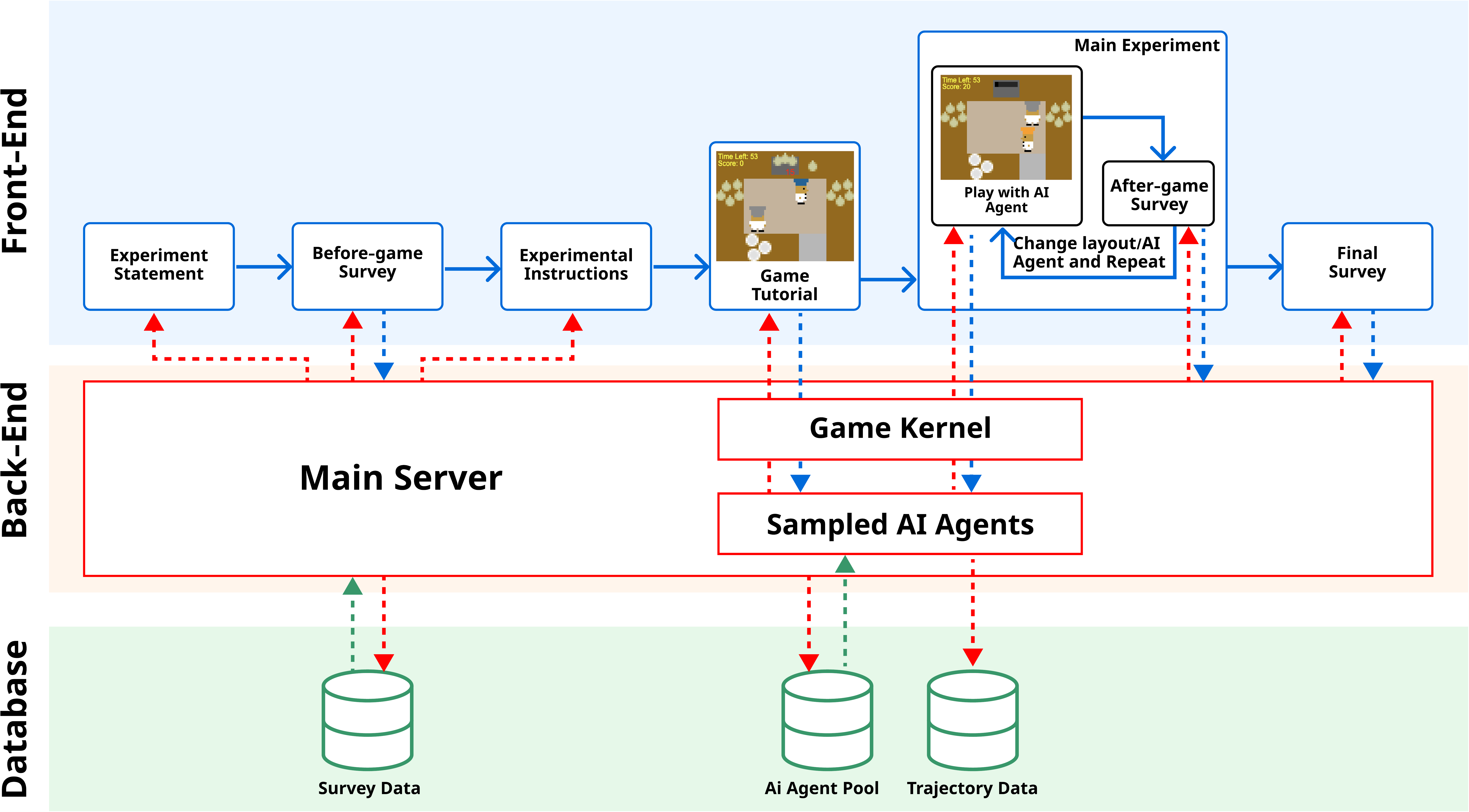}
    \caption{This figure outlines the structure of our human-AI experimental platform, including front-end, back-end, and database components. The front-end, where participants interact, starts with a pre-game process (experiment statement, before-game survey, and experimental instructions) and moves into game stages involving a game tutorial and main experiment with multiple rounds against varying AI agents. A survey follows each round, and a final survey concludes the experiment. All data, including player paths and AI weights, are stored in the database.}
    \label{fig:platform}
\end{figure}

\changes{
As illustrated in Fig.~\ref{fig:platform}, our human-AI experimental platform is composed of three primary elements: the front-end interface, the back-end server, and the database.
The front-end interface is tasked with both rendering the game and recording real-time human keyboard input. Subsequently, at each frame, it engages in communication with the back-end server, transmitting the current game state and keyboard input as serialized data. On receiving this data, the back-end server initially processes it into an observation for the AI agent. Following this, it alters the game environment in response to the actions taken by both the AI agent and the human player. This newly altered game state is then conveyed back to the front-end interface for the human player's perusal.
Simultaneously, the server captures the trajectories of both the human player and the AI agent during each game, storing this information for future analysis. Once the online experiments conclude, these captured trajectories are employed to conclude objective and subjective metrics.}

\paragraph{Pre-experiment Pages.} 
Before beginning the experiment, participants are presented with a page containing terms and conditions (e.g., experiment statement) and must decide if they agree to proceed. We provide the experimental statement used in our experiments as reported in Section~\ref{sec:exp}, which includes information on Experimental Purpose, Procedure, Risks and Discomforts, Costs, and Confidentiality.
If participants agree to the terms and conditions, they will be asked to complete a form with personal information, such as age, sex, level of skills in the game, etc. These data are used to facilitate a more in-depth analysis of the experiment results and will not be used for any purpose other than the experiment itself.

The instruction page then familiarizes participants with the mechanics of the game and the experiment process. Our example instruction page (Fig.~\ref{fig:platform_instruction}) offers details on game settings, world objects, and game controls through both text and images.
Following the instruction page, human players need to participate in a trial game to further practice how to play the game.
Both the instruction page and trial game contribute to participants gaining a preliminary understanding of the experiment.

\paragraph{Gaming and Questionnaire Pages.} 
The next component is the core of the Human-AI experiment, where participants play with randomly sampled AI teammates and complete questionnaires to evaluate the performance of these AI partners. 
Specifically, in our initial platform setup (more details are provided in Section~\ref{sec:exp_human_setting}), participants play with two different AI teammates on five Overcooked layouts, resulting in 10 games per participant. 
After each game, they are required to complete a questionnaire that assesses their cooperative gameplay with the AI partner.
The final questionnaire primarily focuses on coordination with the AI agent with whom they collaborated and the game they just played.
Upon completing all games, participants are directed to the final questionnaire page, where they are asked to provide preliminary feedback on AI agents based on their overall performance, considering factors such as fluency, legibility, and reliability.
In contrast to individual game questionnaires, the purpose of this feedback is to analyze and compare the capacities of different agents.

The Human-AI Experiment Platform is designed for seamless customization, enabling users to easily modify various aspects of the system, such as statements, instructions, questionnaires, and game settings. Most customizations can be achieved simply by editing configuration files, streamlined for users to adapt the platform to their specific needs.

\subsection{\changes{Experimental Scale}}
\label{sec:agent_eval}

\changes{In our effort to thoroughly examine the performance of human-AI coordination in a zero-shot scenario, we advocate for a systemically designed experimental scale. 
This scale incorporates an expansive range of subjective metrics that appraise individual games involving AI agents while providing a holistic comparison across multiple coordinated players.
The metrics utilized extend beyond the superficial, capturing elements such as intention, contribution, and team dynamics during human-AI interactions. Furthermore, the evaluation process encompasses additional aspects such as fluency, participant preferences, and comprehension across all collaborated agents.
More specifically, participants are requested to participate in after-game surveys and final-game surveys to assess the performance of AI agent collaboration. Each of these surveys comprises three questions that use a 7-point Likert scale, illustrated in Table~\ref{tab:after_game}. The selected survey questions for our experiment draw heavily from a pool of analogous questions found in \cite{hoffman2019evaluating}.
For the after-game surveys, the 7-point Likert scales are arranged with 1 denoting "strongly disagree" and 7 indicating "strongly agree". In the final experiment questionnaire, ratings from 1 to 7 reflect preferences from "strongly favoring the first agent" to "strongly favoring the second agent", for instance, a rating of 4 symbolizes an absence of preference. The rating system adopts a star-assigning format, with the rules clearly stated at the onset of each questionnaire.
}
\begin{table}[ht]
    \centering
        \caption{Evaluation statements for the after-game and final questionnaires consist of three statements each. After-game statements aim to assess the cooperative experience with the AI agent in a single game, while final statements require participants to compare the two AI partners and rate their performance after all games have been completed.}
    \label{tab:after_game}
    \begin{tabular}{c|c|c}
    \hline
           \textbf{Type}&\textbf{Index}& \textbf{Scale statement} \\
         \hline
       \multirow{3}{*}{\textbf{After-game}} & Q1  &  The agent and I have good teamwork.\\
       &Q2  &  The agent is contributing to the success of the team. \\
       &Q3  &  I understand the agent's intentions.\\
       \hline
       \multirow{3}{*}{\textbf{Final}} & Q1  &  Which agent cooperates more fluently?\\
       &Q2  &  Which agent did you prefer playing with? \\
       &Q3  &  Which agent did you understand with?\\
       \hline
    \end{tabular}
\end{table}

\section{Experiments}
\label{sec:exp}
We carry out a series of experiments involving AI agents, human proxies, and human players to assess the performance of \algo compared to the baseline methods when collaborating with partners in zero-shot settings. Our experiments focus on the following research questions (RQ):

\begin{itemize}
    \item RQ1: What is the optimal balance between individual and socially compatible objectives?  (Section \ref{exp:ratio})
    \item RQ2: How do the \algo and baseline agents perform when cooperating with agents of varying skill levels in a zero-shot setting? (Section \ref{exp:agent})
    \item RQ3: Does \algo effectively address the issue of cooperative incompatibility? (Section \ref{casestudy})
    \item \changes{RQ4: How do each modules within \algo contribute to its overall performance? (Section \ref{exp:ablation})}
    \item \changes{RQ5: Which solver exhibits superior performance? (Section \ref{exp:solver})}
    \item RQ6: How do the \algo and baseline agents perform when working with humans in a zero-shot setting, and which algorithm do human users prefer? (Section \ref{exp:human})
\end{itemize}


\subsection{Evaluation of Combining Objectives' Effectiveness}
\label{exp:ratio}
\begin{figure}[t!]
    \centering
    \includegraphics[width=0.9\linewidth]{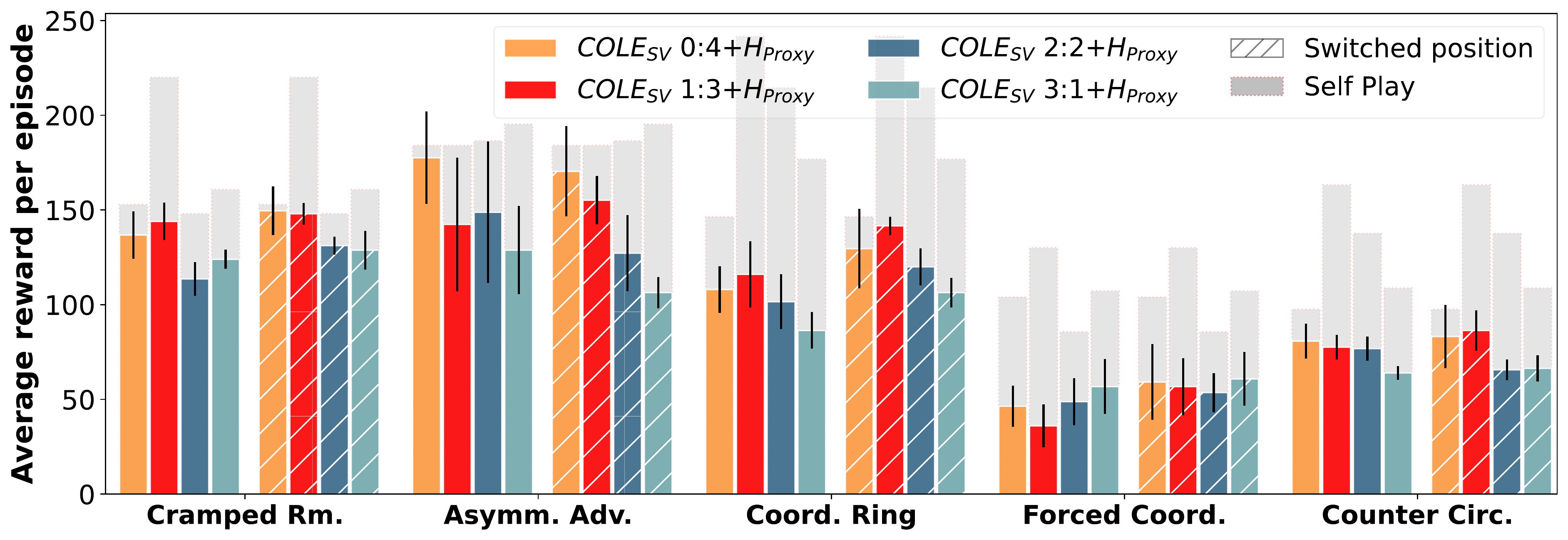}
    \caption{{The result of the combining objectives' effectiveness evaluation.}
    Mean episode rewards over 400 timesteps trajectories for \algo s with different objective ratios 0:4, 1:3, 2:2, and 3:1, paired with the unseen human proxy partner $H_{proxy}$.
    \changes{The ratios 0:4, 1:3, 2:2, and 3:1 denote varying proportions between individual and cooperative compatible objectives.}
    The gray bars behind present the rewards of self-play.
    }
    \label{fig:exp_ablation}
\end{figure}

We construct evaluations with different ratios between individual and cooperative compatible objectives, such as 0:4, 1:3, 2:2, and 3:1. 
These studies demonstrate the effectiveness of optimizing both individual and cooperative incompatible goals. 

\subsubsection{Experimental Setting}
We divided each training batch into four parts, the ratio indicating the proportion of data generated by self-play and data generated by playing with strategies from the cooperative incompatible distribution. 
We omitted the 4:0 ratio as it would result in the framework degenerating into self-play.

\subsubsection{Results}
Fig.~\ref{fig:exp_ablation} shows the mean rewards of episodes over 400 time steps of gameplay when paired with the unseen human proxy partner $H_{proxy}$ \cite{HARL}. 
We found that \algo with ratios 0:4 and 1:3 achieved better performance than the other ratios. 
In particular, \algo, with a ratio of 1:3, outperformed the other methods in the Cramped Room, Coordination Ring, and Counter Circuit layouts. 
On the Forced Coordination layout, which is particularly challenging for cooperation due to the separated regions, all four ratios performed similarly on average across different starting positions.  
Interestingly, \algo with only the cooperative compatible objective (ratio 0:4) performed better on the Asymmetric Advantages and Forced Coordination layouts when paired with the human proxy partner.
Effectiveness evaluations indicate that the combination of individual and cooperatively compatible objectives is crucial to improving performance with unseen partners.
In general, we choose the ratio of 1:3 as the best choice.

We further visualize the trajectories produced by \algo 1:3 and 0:4 with human proxy and expert partners in Overcooked on our demo page.
Fig.~\ref{fig:case_study} presents three screenshots of the \algo 0:4 model (blue player) that collaborates with one of the expert partners, the PBT model (green player). 
The case illustrates the importance of the individual objects in zero-shot coordination with expert partners. 
Frame A is a screenshot taken at 53s when the two players start to impede each other. 
The PBT model has taken the plate and wants to load and serve the dish. 
The blue player wants to take the plate but does not know how to change the objective to allow the green player to load the dish. 
After blocking for about 11s, the blue player starts to move and lets the green player go to the pots (Frame B). 
However, the process is not smooth and takes 7s to reach Frame C. 
This phenomenon does not occur in \algo 1:3 coordination with expert partners, which shows that including individual objectives might improve the cooperative ability with expert partners.

\begin{figure}[h]
    \centering
\includegraphics[width=0.95\linewidth]{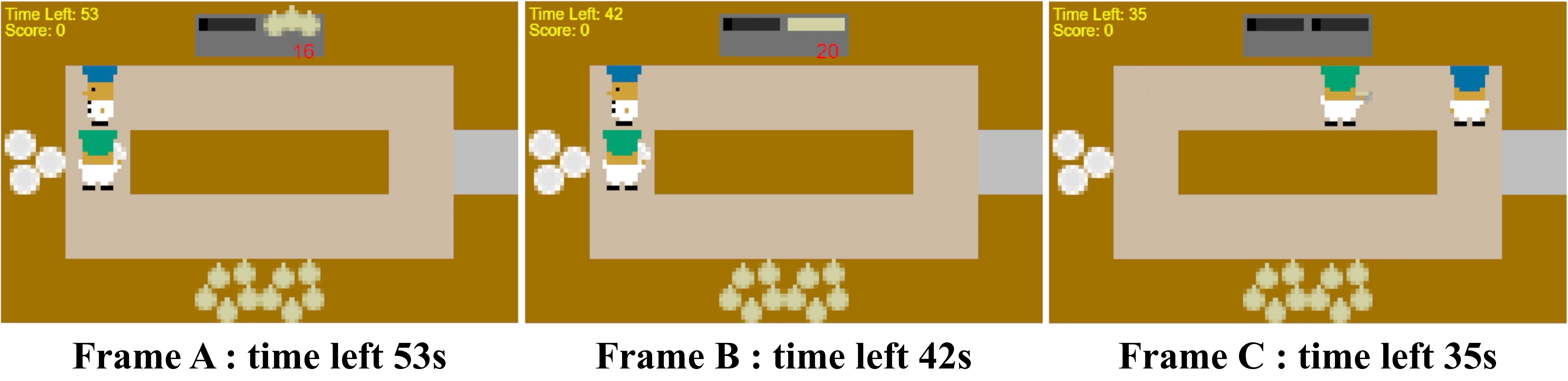}
\caption{
{Trajectory snapshots of the \algo 0:4 model (blue) with one of the expert partners - PBT model (green).}
}
    \label{fig:case_study}
\end{figure}

\subsection{Evaluation with Human Proxy and AI Agents}
\label{exp:agent}

To thoroughly assess the ZSC ability, we evaluated the algorithms with unseen human proxy and expert partners. 
We compare our method with other methods, including self-play~\cite{SP,HARL}, PBT~\cite{PBT,HARL}, FCP~\cite{FCP}, and MEP~\cite{MEP}, all of which use PPO~\cite{PPO} as the RL algorithm. 
We use the human proxy model $H_{proxy}$ proposed in~\cite{HARL} as human proxy partners and the models trained with baselines and \algo as expert partners. 

\subsubsection{Experimental Setting}
\label{subsec:settings}

We adopted two sets of evaluation protocols for the evaluation. 
The first protocol involves playing with a trained human model $H_{proxy}$ trained in behavior cloning.
Due to the quality and quantity of human data used for behavior cloning to train the human model is limited, the capabilities of the human proxy models are limited. 
Therefore, we use an additional evaluation protocol to coordinate with unseen expert partners. 
We selected the best models of our reproduced baselines and \algo 0:4 and 1:3 as expert partners.
The mean of the rewards is recorded as the performance of each method in collaborating with expert teammates. 
Appendix~\ref{appendix:cole} and Appendix~\ref{appendix:base} give details of the implementation of \algo and baselines.

\subsubsection{Results with Human Proxy and AI Agents}
\label{different_levels}
\begin{figure}[t!]
    \centering
    \includegraphics[width=0.95\linewidth]{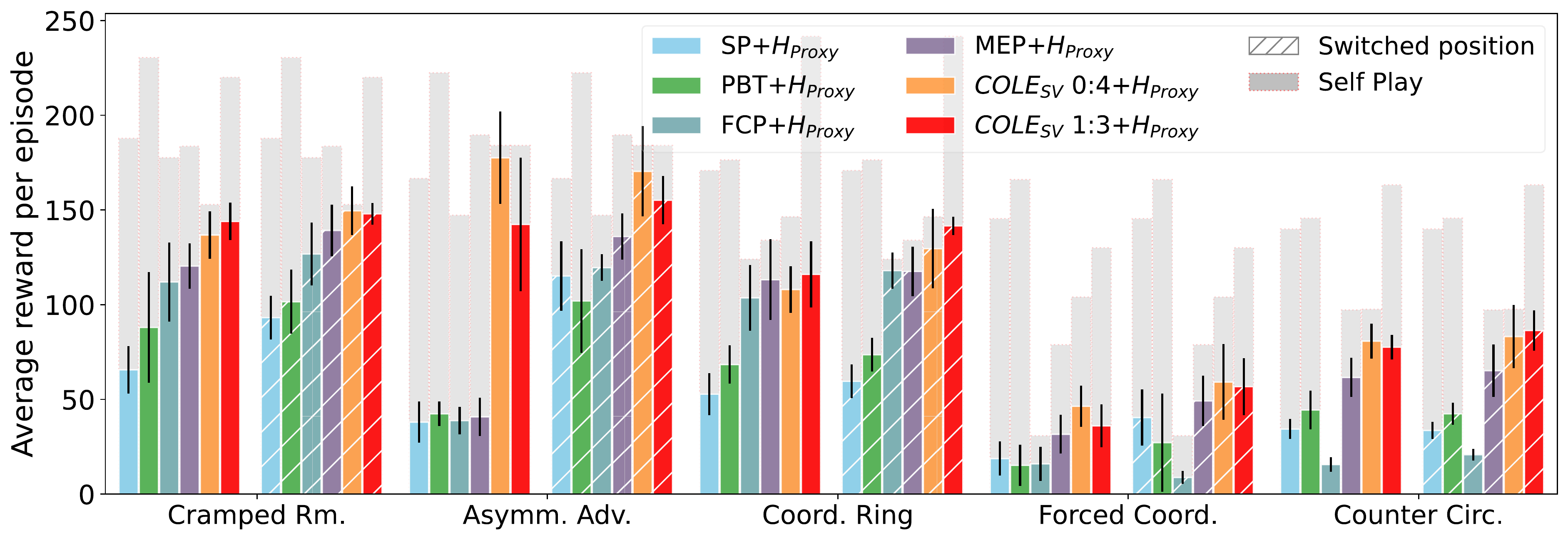}
    \caption{{Performance with the human proxy partner.}
    The performance of \algo with human proxy partners is presented in terms of mean episode rewards over 400 timesteps trajectories for different objective ratios of 0:4 and 1:3, when paired with the unseen human proxy partner $H_{proxy}$. 
    \changes{The ratios 0:4 and 1:3 denote varying proportions between individual and cooperative compatible objectives.}
    The results include the mean and standard error over five different random seeds. The gray bars indicate the rewards obtained when playing with themselves; the hashed bars indicate the performance when starting positions are switched.
    }
    \label{fig:main_exp}
\end{figure}

Fig.~\ref{fig:main_exp} presents the performance of SP, PBT, MEP, and \algo with 0:4 and 1:3 when cooperating with human proxy partners. 
We observed that different starting positions on the left and right in asymmetric layouts resulted in significant performance differences for the baselines. 
For example, in the Asymmetric Advantages, the cumulative rewards of all baselines in the left position were nearly one-third of those in the right position. 
On the contrary, \algo performed well at the left and right positions.

As shown in Fig.~\ref{fig:main_exp}, \algo outperforms other methods in all five layouts when paired with the human proxy model. 
Interestingly, \algo 0:4 with only the cooperatively compatible objective achieves better performance than \algo 1:3 on some layouts, such as Asymmetric Advantages. 
However, the self-play rewards of \algo 0:4 are much lower than \algo 1:3 and even other baselines. 
The objective function of \algo 0:4 consists of only cooperative compatible term $\mathbb{E}_{p\sim \phi}\rvw(s_n,p)$.
We believe that focusing solely on cooperative compatible objectives may lead to learning stagnation when collaborating with low-performing partners sampled from distribution $\phi$. 
Consequently, the self-play rewards may be limited.
Furthermore, the performance with unseen experts of \algo 0:4 as shown in Table~\ref{tab:exp_expert}, is sometimes lower than the baselines.

\begin{table}[t!]
\caption{Performance with expert partners. Mean episode rewards over 1 min trajectories for baselines and \algo with ratio 0:4, 1:3.
    \changes{The ratios 0:4 and 1:3 denote varying proportions between individual and cooperative compatible objectives.}
  Each column represents a different expert group, in which the result is the mean reward for each model playing with all others.
  }
\label{tab:exp_expert}
\begin{center}
\resizebox{0.8\linewidth}{!}{%
\begin{sc}
    \begin{tabular}{lcccccc}
\toprule
\multirow{2}{*}{\textbf{Layout}} &\multirow{2}{*}{\textbf{Ratio}} & \multicolumn{4}{c}{\textbf{Baselines}}  &\multirow{2}{*}{\textbf{COLEs}}\\
\cline{3-6}\
 && \textbf{SP} & \textbf{PBT} & \textbf{FCP} & \textbf{MEP} &  \\
\midrule
\multirow{2}{*}{\textbf{Cramped Rm.}} &0:4&
153.00 & 198.50  & {199.83 } & 178.83 & 169.76 \\&1:3& 
165.67 & 209.83 & 207.17 & 196.83 & \textbf{212.80}\\
\hline
\multirow{2}{*}{\textbf{Asymm.Adv.}} &0:4&
108.17  & 164.83 & 175.50 & 179.83& \textbf{182.80}
\\&1:3&
108.17 & 161.50 & 172.17 & {179.83} & 178.80\\
 \hline
\multirow{2}{*}{\textbf{Coord. Ring}}&0:4&
132.00 & 106.83 & {142.67} & 130.67  & 118.08
\\&1:3&
133.33 & 158.83 & 144.00  & 124.67 &\textbf{166.32}\\
 \hline
\multirow{2}{*}{\textbf{Forced Coord.}} &0:4&
~~58.33 & ~~61.33 & ~~50.50  & ~~{79.33} &  ~~46.40\\&1:3&
~~61.50  & ~~70.33 & ~~62.33  & ~~38.00  &~~\textbf{86.40}\\
 \hline
\multirow{2}{*}{\textbf{Counter Circ.}}&0:4&
~~44.17  & ~~48.33 & ~~60.33& ~~21.33 & ~~{90.72}\\
&1:3&
~~65.67  & ~~64.00  & ~~46.50  & ~~76.67  &  \textbf{105.84}
\\
\bottomrule
\end{tabular}
\end{sc}
}
\end{center}
\end{table}
Table~\ref{tab:exp_expert} presents the outcomes of each method when cooperating with expert partners. 
Each column in the table represents different expert groups, including four baselines and one \algo with a ratio of 0:4 or 1:3. 
The last column, labeled ``COLEs'', represents the mean rewards of the corresponding \algo when working with other baselines. 
The table displays the mean cumulative rewards of each method when working with all other models in the expert group. 
The results indicate that \algo 1:3 outperforms the baselines and \algo 0:4, except in the layout of Asymmetric Advantages.
In the Asymmetric Advantages, \algo 0:4 only achieved a four-point victory over \algo 1:3, which can be considered insignificant considering the margin of error. 
In the other four layouts, the rewards obtained by \algo 1:3 while working with expert partners are significantly higher than those of \algo 4:0 and the baselines.

\changes{
Our findings indicate that \algo 1:3 exhibits superior adaptive capacity when dealing with partners of expert levels, emphasising individual objectives is key to achieving zero-shot coordination with expert partners. 
To summarize, \algo 1:3 manifests enhanced robustness and versatility in real-world environments, making it apt for collaboration with partners across a spectrum of expertise levels.
Hence, in the forthcoming experiments and human-AI interaction studies, we will implement \algo 1:3 as our definitive agent, tailored to adapt to human players of diverse skill levels.}

\subsection{Effectively Conquer Cooperative Incompatibility}
\label{casestudy} 
\begin{figure}[t!]
\centering    \includegraphics[width=0.7\linewidth]{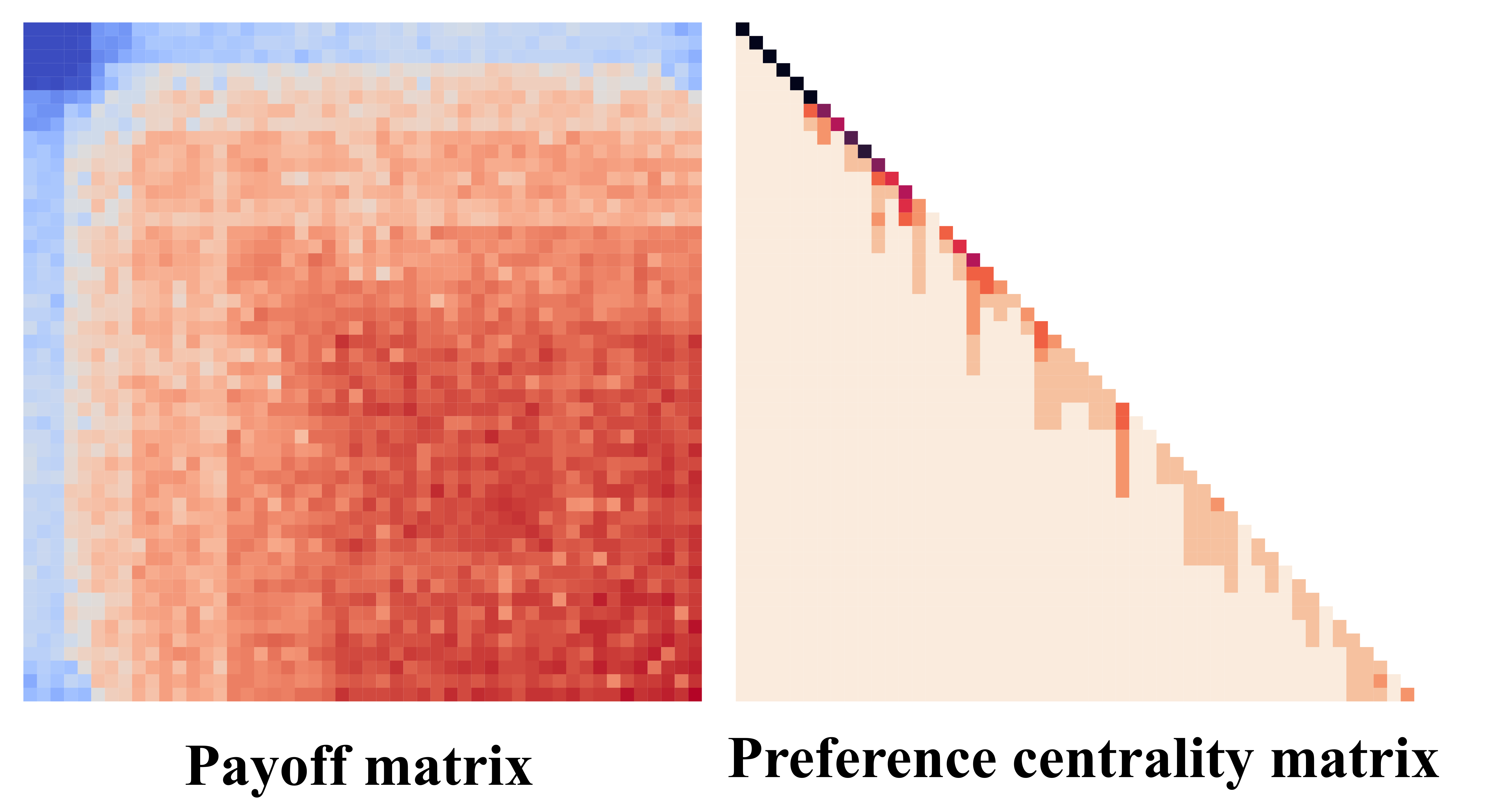}
    \caption{
    {The learning process analysis of \algo 1:3.}
A deeper shade of red in the payoff matrix signifies  higher utility, while the darker-colored element on the right represents lower centrality. Clustering of darker-colored areas around the diagonal on the right indicates that the new strategy adopted in each generation is preferred by most strategies, thus overcoming the cooperative incompatibility.
}
\label{fig:cole_analysis}
\end{figure}
To We analyze the learning process of \algo, which shows that our method overcomes cooperative incompatibility. 

In our analysis of the learning process of \algo 1:3 in the Overcooked environment, as shown in Fig.~\ref{fig:cole_analysis}, we observe that the method effectively overcomes the problem of cooperative incompatibility. 
The figure on the left in Fig.~\ref{fig:cole_analysis} shows the payoff matrix of 50 uniformly sampled checkpoints during training, with the upper left corner representing the starting point of training. 
Darker red elements in the payoff matrix indicate higher rewards. 
The figure on the right displays the centrality matrix of preferences, which is calculated by analyzing the learning process. 
Unlike the payoff matrix, the darker elements in the centrality matrix indicate lower values, indicating that more strategies prefer them in the population. 
As shown in the figure, the darker areas cluster around the diagonal of the preference centrality matrix, indicating that most of the others prefer the updated strategy of each generation. 
Thus, we can conclude that our proposed \algo effectively overcomes the problem of cooperative incompatibility.

\begin{figure}[ht!]
    \centering
    \includegraphics[width=0.95\linewidth]{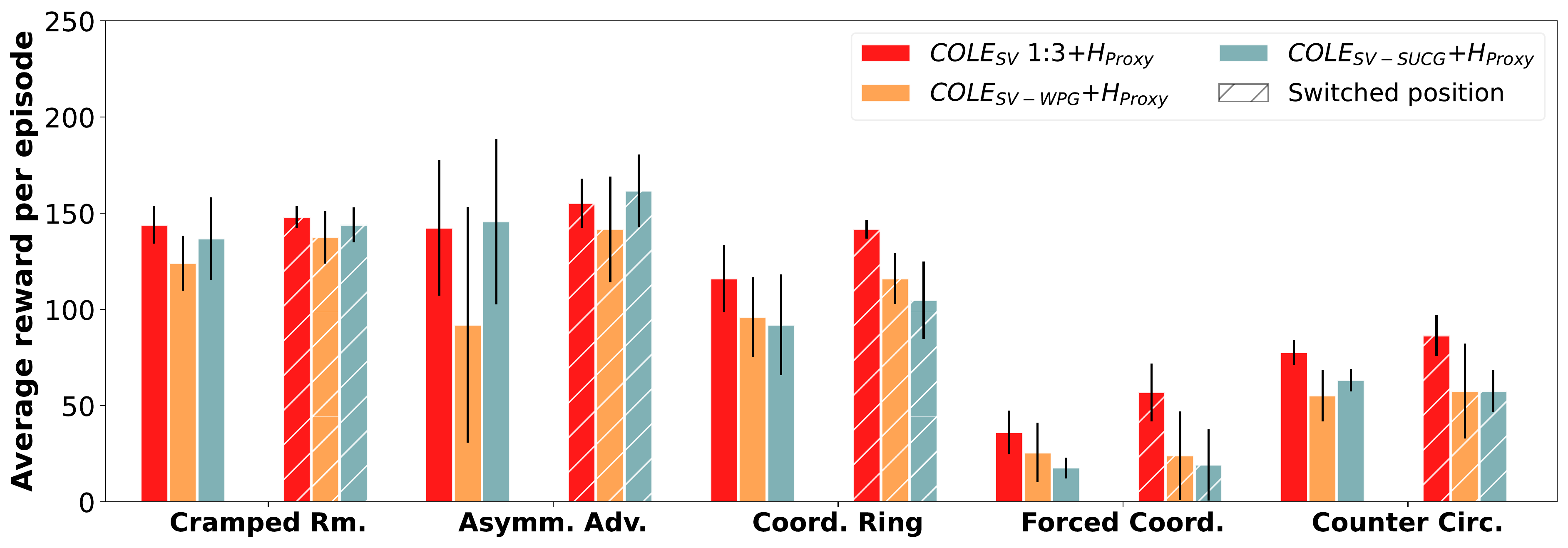}
    \caption{
    \changes{Performance comparison of module effectiveness in \algo performance.
    The performance is presented in terms of mean episode rewards over 400 timesteps trajectories, when paired with the unseen human proxy partner $H_{proxy}$.
  $\text{COLE}{\textit{\tiny{SV-WPG}}}$ excludes the WPG element (Eq.~\ref{eq:wpg}) while computing the Shapley value. 
  $\text{COLE}{\textit{\tiny{SV-UCGG}}}$ eliminates the SUCG component (Eq.~\ref{eq:SUCG}) while sampling training partners.
    The results include the mean and standard error over five different random seeds. The gray bars indicate the rewards obtained when playing with themselves; the hashed bars indicate the performance when starting positions are switched.}}
    \label{fig:ablation}
\end{figure}

\subsection{\changes{Ablation Study}}
\label{exp:ablation}

\changes{
In this section, we aim to investigate the efficiency of each component within our proposed algorithm. 
Refer to Fig.~\ref{fig:ablation} where we compare the performance of our algorithm against three ablated models namely, $\text{COLE}{\textit{\tiny{SV-WPG}}}$, $\text{COLE}{\textit{\tiny{SV-UCGG}}}$, thereby study the individual impact of these specific components.
Specifically, $\text{COLE}{\textit{\tiny{SV-WPG}}}$ excludes the WPG element (Eq.~\ref{eq:wpg}) while computing the Shapley value. 
Therefore, in the case of $\text{COLE}{\textit{\tiny{SV-WPG}}}$, the computation of the coalition value employs the average of all utilities within the coalition. That is, $v(C)=\frac{1}{n^2} \sum_{i\in C}\sum_{j\in C}\rvw(i,j)$, where $C$ represents a coalition within the full coalition set $\gN$.
Additionally, we examine the effect of the SUCG element (Eq.~\ref{eq:SUCG}) on the performance of our proposed algorithm during the sampling of training partners. This model is referred to as $\text{COLE}{\textit{\tiny{SV-UCGG}}}$.
}

\changes{Fig.~\ref{fig:ablation} illustrates the performance in terms of average episode rewards over trajectories of 400 timesteps, when coupled with the unseen human proxy partner $H_{proxy}$. This data presents both the mean and the standard error across five unique random seeds. The gray bars reveal the rewards yielded when the models compete against themselves, while the hashed bars represent the performance when the starting positions are alternated.
It is evident from the figure that the WPG element (Eq.~\ref{eq:wpg}) significantly contributes to the performance when interacting with the human proxy partner, particularly noticeable in the Asymm. Adv. layout.
Furthermore, as depicted in Fig.~\ref{fig:ablation}, all baseline models demonstrate subpar performance when coordinating with the human proxy model in position 0 of the Asymm. Adv. layout.
The computation of the Shapley value employing the WPG is fundamental to resolving the asymmetry, as depicted in Fig.~\ref{fig:main_exp}. 
We think the reason behind it is that WPG can provide a just evaluation of each position's strategic cooperative ability.
In relatively straightforward layouts such as Cramped Rm. and Asymm. Adv., the algorithm displays performance comparable to the ablated models $\text{COLE}{\textit{\tiny{SV-WPG}}}$, $\text{COLE}{\textit{\tiny{SV-UCGG}}}$. However, in the remaining three complex layouts, there is a significant enhancement in the \algo's performance.
}

\begin{table}
    \centering
    \resizebox{\linewidth}{!}{
    \begin{tabular}{ccccccc}
    \toprule
    \multirow{2}{*}{\textbf{Methods}} & \multirow{2}{*} {\textbf{Position}}&\multicolumn{5}{c}{\textbf{Layouts}}\\
\cline{3-7} &&
    \bf{CRAMPED RM.} & \bf{ASYMM. ADV.} & \bf{COORD. RING} & \bf{FORCED COORD.} & \bf{COUNTER CIRC.} \\
    \midrule
    \multirow{3}{*}{\bf{\algoR}} &\bf{L} & 131.20 (14.18) &158.40 (31.66) &80.80 (9.60) &21.60 (10.91) &63.20 (9.93) \\
    \cline{2-7}& \bf{R} & 135.20 (14.40) &156.80 (19.17) &116.00 (8.39) &27.20 (26.70) &52.00 (10.73) \\ 
    \cline{2-7}& \bf{AVG} & 133.20 (14.43) &157.60 (26.18) &98.40 (19.77) &24.40 (20.59) &57.60 (11.76) \\
    \midrule
    \multirow{3}{*}{\bf{\algo}} & \bf{L} & 144.00 (9.80) &142.40 (35.29) &116.00 (17.53) &36.00 (11.31) &77.60 (6.50) \\
    \cline{2-7}& \bf{R} & 148.00 (5.66) &155.20 (12.75) &141.60 (4.80) &56.80 (15.05) &86.40 (10.61) \\
    \cline{2-7} & \bf{AVG} & 146.00 (8.25) &148.80 (27.29) &128.80 (18.14) &46.40 (16.89) &82.00 (9.84) \\
    \bottomrule
    \end{tabular}
    }
    \caption{\changes{Performance comparison of \algo and \algoR on five layouts. 
    Mean scores and standard errors (denoted in parentheses) are measured over five different seeds. 
    In the Position column, L and R signify respective initial positions on the left and right of the layout. The term AVG stands for the average score obtained from both sides.
    \algo has demonstrated significantly superior performance compared to \algoR across all layouts, with the exception of the Asymm. Adv. layout. Within this specific Asymm. Adv. setting, \algoR outperformed in the left position and exhibited a comparable score in the right position.
    }
    }
    \label{tab:R_SV}
\end{table}

\subsection{\changes{Comparison of \algo and \algoR}}
\label{exp:solver}
\changes{Table~\ref{tab:R_SV} presents a comparative analysis of the coordination payoffs achieved by two practical algorithms, \algo and \algoR, across five different layouts.
The results, acquired via five distinct seed values with the human proxy model $H_{proxy}$, feature standard errors, which are indicated within parentheses. The column named `position' designates the diverse initial positions of the two practical algorithms respectively; here, `L' stands for left, `R' for right, while `AVG' signifies the average derived from the two positions.
}

\changes{
As demonstrated in Table~\ref{tab:R_SV}, \algo, utilizing the Shapley value as the core, substantially outperforms \algoR in the Cramped Rm., Coord. Ring, Forced Coord., and Counter Circ. layouts. For the latter, more challenging layouts (Coord. Ring, Forced Coord., and Counter Circ.), the average scores of \algo with the $H_{proxy}$ model have shown approximately 30\%, 91\%, and 43\% improvements over \algoR.
Interestingly, despite the two algorithms demonstrating similar performance on the right-hand starting position of the Asymm. Adv. layout, \algoR exhibits superior performance when initiated from the left-hand starting position.
}

\subsection{Evaluation with Human Players}
\label{exp:human}

We recruited 130 students as participants from universities in different majors, and each provided written informed consent.
Each participant received a reward of 50 CNY for their participation. 
To motivate them to be more engaged and attentive, we established a goal (a number of rewards) for each layout. 
Participants who successfully reached the goal were awarded an additional bonus of 5 CNY for each layout. 
Experiments were conducted online where each participant completed their own experiment on a certain web page using a computer, and each experiment costs about 15 to 20 minutes. 
We do not collect any personally identifiable information, and no significant risk to participants was expected.

\subsubsection{Experimental Setting}
\label{sec:exp_human_setting}
Participants were first acquainted with the experiment and the rules of Overcooked, with details provided in Section \ref{sec:cole_platform}. Each participant played a sequence of 5 pairs of games, totaling 10 rounds, with 2 agents (one being COLE, while the other was randomly selected from the four agents in section \ref{subsec:settings}). The experiment employed the in-group setting for each baseline and \algo pair. We also randomized the order of agents within each pair to account for skill differences arising from the inner order variance.

The sequence of five different layouts remained consistent, and we did not compare objective scores across layouts. The five pairs corresponded to five distinct layouts, arranged in the same order for all experiments (1. Cramped Rm., 2. Asymm. Adv., 3. Coord. Ring, 4. Forced Coord., 5. Counter Circ.). Each pair featured one round with the \algo agent and one round with the other agent. Under these conditions, we can assume that each participant possesses the same skill level and prior knowledge when encountering the same layouts.

During the experiment, participants were unaware of the specific algorithm names. To ensure fairness, only the color of the chef's hat in the game was used to differentiate between the two algorithms. Each game lasted for 1 minute (approximately 400 steps).

\subsubsection{Results of Human Evaluation}
\label{sec:exp_human_results}

\begin{figure}[t!]
\centering
\subfigure[Intention]{
\begin{minipage}[t]{0.5\linewidth}
\centering
\includegraphics[width=.96\linewidth]{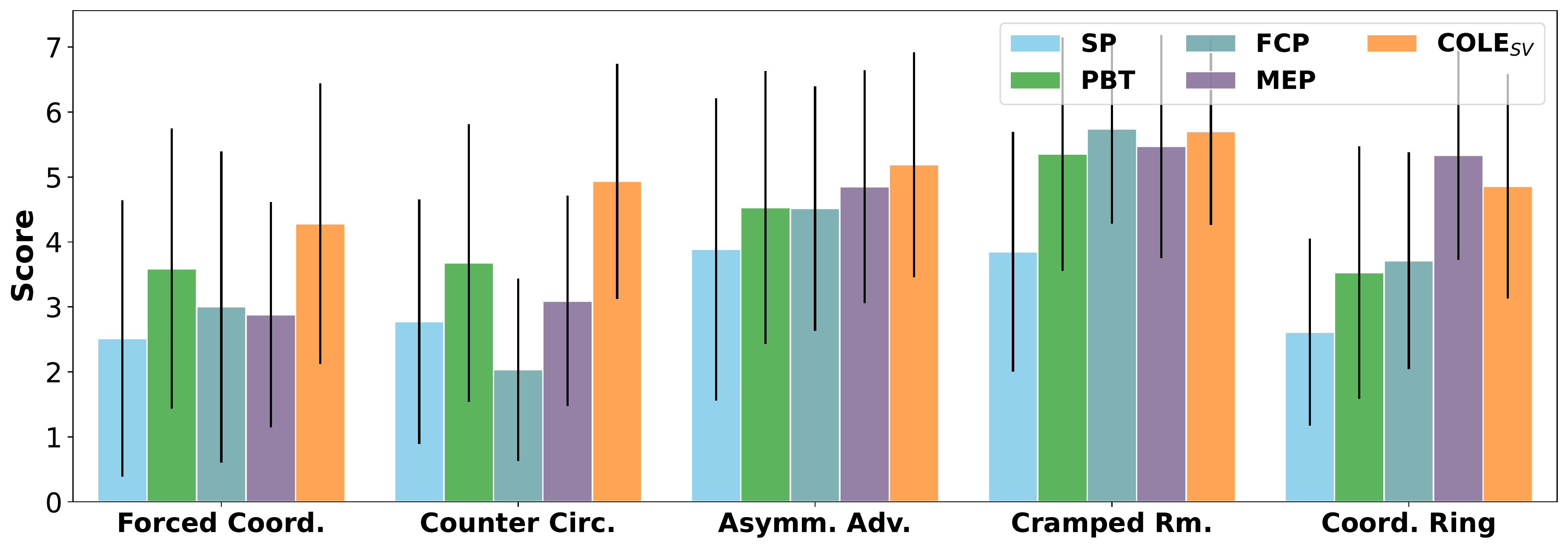}
\end{minipage}%
}%
\subfigure[Contribution]{
\begin{minipage}[t]{0.5\linewidth}
\centering
\includegraphics[width=.96\linewidth]{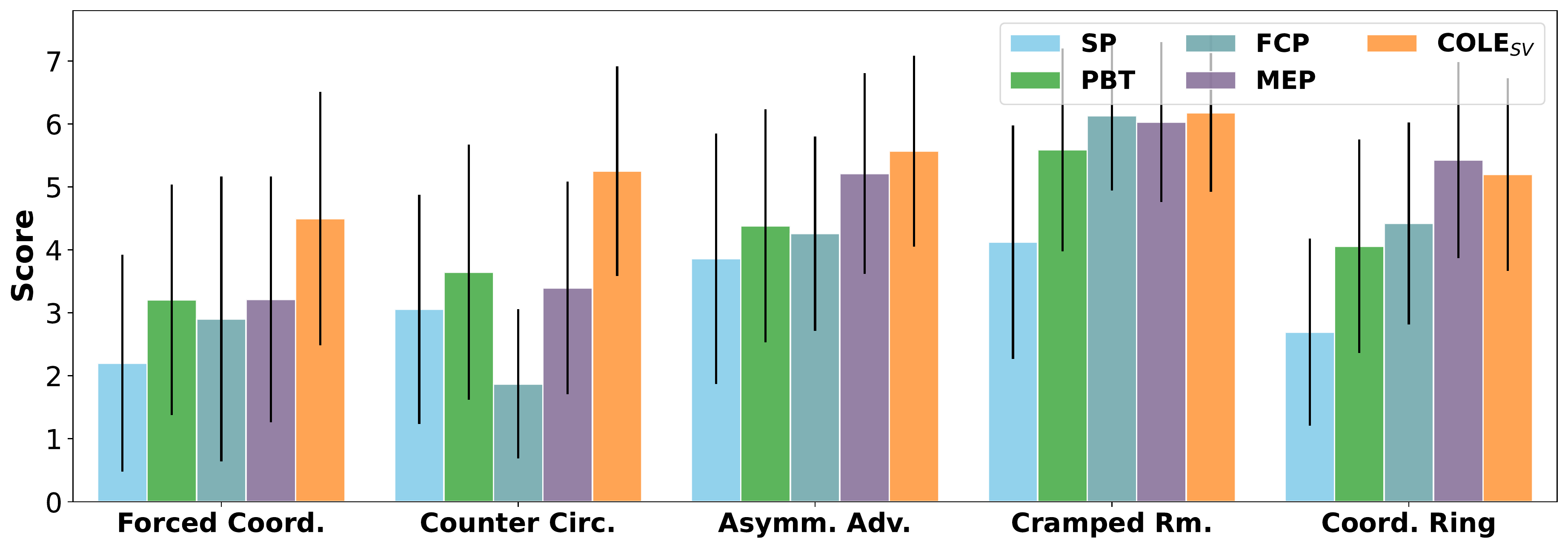}
\end{minipage}
}%

\subfigure[Teamwork]{
\begin{minipage}[t]{0.5\linewidth}
\centering
\includegraphics[width=.96\linewidth]{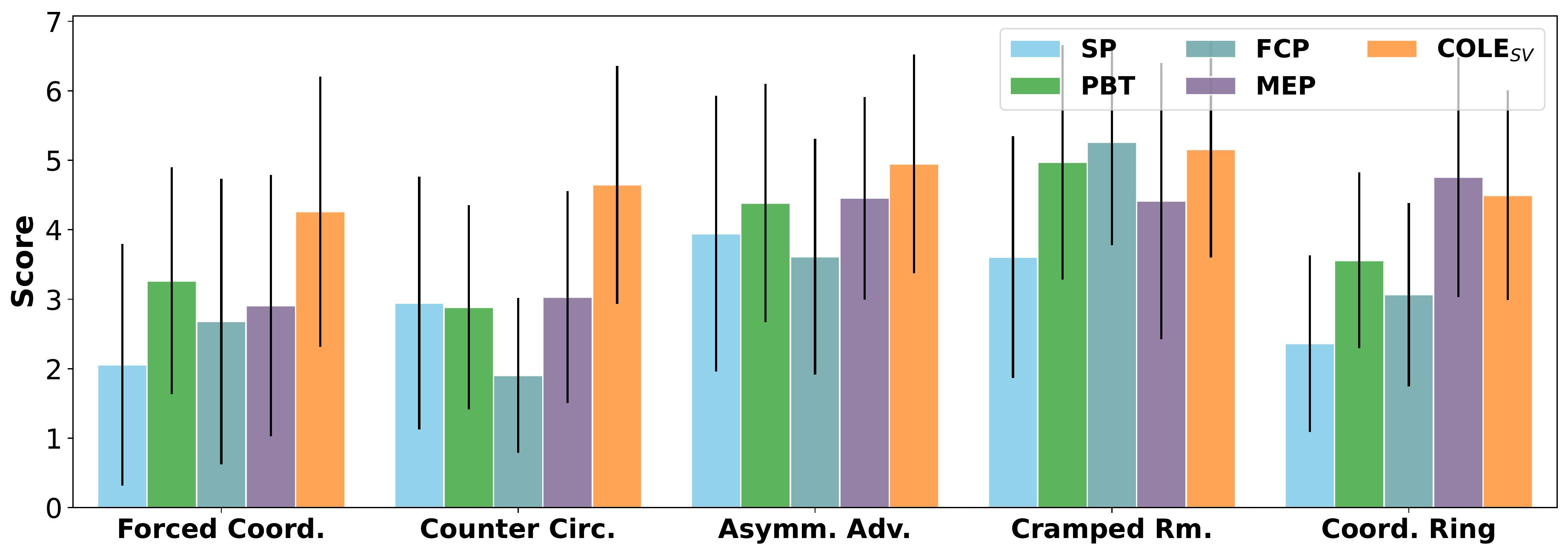}
\end{minipage}
}%
\subfigure[Rewards]{
\begin{minipage}[t]{0.5\linewidth}
\centering
\includegraphics[width=.96\linewidth]{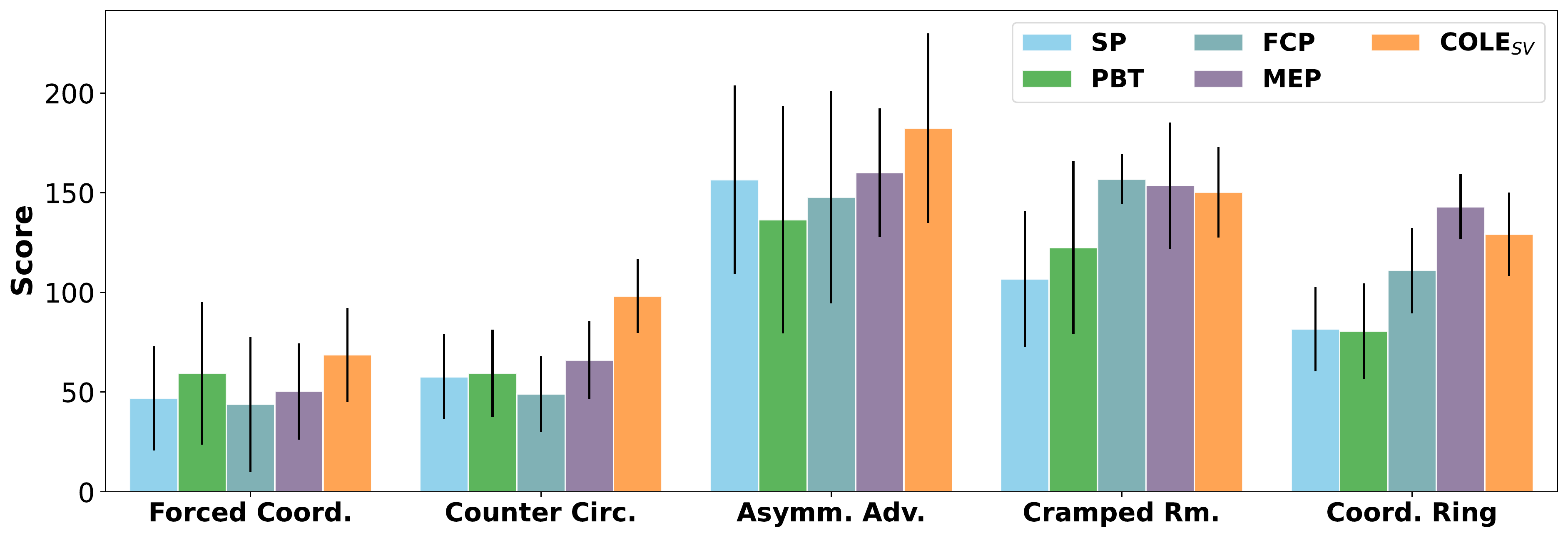}
\end{minipage}%
}%
\label{fig:sub_metrics_human}
\centering
\caption{
Comparison of subjective evaluation scores and rewards between our \algo method and baseline algorithms across five different layouts. ``Intention'', ``Contribution'' and ``Teamwork'' are shorts for ``I understand the agent's intentions'', ``The agent is contributing to the success of the team'', and ``the agent and I have good teamwork'', respectively. Fig. (d) depicts the average rewards comparison. Our \algo method obviously outperforms the baselines in Forced Coordination, Counter Circulation, and Asymmetric Advantage. The first three layouts necessitate cooperation for higher rewards. In the simplest layout, Cramped Room, PBT, FCP, MEP, and \algo all achieve comparably high scores when playing with humans. In the Coordinated Ring layout, \algo's performance is slightly below MEP but surpasses the other algorithms.}
\end{figure}

In this section, we discuss the findings from our Human-AI experiment. We initially recruited 148 participants from Shanghai Jiaotong University to participate in the study. However, we had to exclude 18 questionnaires from our analysis due to issues such as incomplete responses or negative scoring (e.g., awarding 0 points to all aspects). Thus, our final sample consisted of 130 valid data entries.

The demographic breakdown of the 130 valid participants is as follows: 91 males and 39 females. In terms of prior experience with the Overcooked video game, 70 participants had never played it, 28 had played it but did not consider themselves skilled, 25 regarded themselves as intermediate players, and 7 viewed themselves as experts.

As described in Section~\ref{subsec:settings}, all participants are required to fill in after-game and final questionnaires. 
The after-game questionnaires are about the subjective evaluation of the playing between participants and one AI agent. 
The subjective evaluation needs human players to score three questions from 1 to 7, ``the agent and I have good teamwork'' (short as teamwork), ``the agent is contributing to the success of the team'' (short as contribution), and ``I understand the agent's intentions''  (short as intention). 
As illustrated in Fig.~\ref{fig:sub_metrics_human}, our proposed \algo surpasses all baseline algorithms in three layouts that require cooperation (Forced Coord., Counter Circ., and Asymm Adv) based on subjective evaluations of intention, contribution, teamwork, and the objective metric of game rewards. 
Additionally, human players perceive that all methods, except SP, exhibit similar performance in the Crampped Rm layout. 
In Crampped Rm., both players are situated within a small rectangular area, enabling them to achieve high scores even with minimal cooperation.

In the {Coord. Ring} layout, \algo marginally underperforms MEP yet surpasses other baseline algorithms. Analyzing human and MEP co-play trajectories revealed MEP's enhanced performance in human subjective experiments arises from its consistent counterclockwise strategy on the layout. This allows humans to effortlessly adapt to the MEP agent during human-AI experiments by selecting the opposite direction. This predictability leads to misconceptions in questionnaire responses regarding AI's contribution, intention, and teamwork. Participants may perceive a better grasp of MEP's intentions and a more substantial contribution from the agent, but this is primarily due to its easily adaptable and fixed strategy. Nonetheless, when MEP collaborates with diverse partners lacking human-level intelligence, its performance declines, as shown in Fig.\ref{fig:main_exp}. 
In contrast, COLE's strategies exhibit greater diversity and are less predictable than MEP's. COLE adapts its routes based on varying situations, which can lead to more conflicts. Consequently, in human subjective evaluations, COLE slightly underperforms compared to MEP.
Trajectory visualizations of humans with MEP and COLE are available on our demo page.

\begin{table}[t]
    \centering
    \begin{tabular}{c|c|c|c|c}
    \hline

    \multirow{2}{*}{\textbf{Metrics}}&\multicolumn{4}{c}{\textbf{\algo v.s. Baselines }}\\
\cline{2-5}

 & \textbf{SP} & \textbf{PBT} & \textbf{FCP} & \textbf{MEP}\\
    \hline
         \textbf{Fluency} &$0.87$   &$0.71$   &$0.78$   &$0.60$   \\
         \textbf{Preference} &$0.88$   &$0.70$   &$0.85$   &$0.68$    \\ 
         \textbf{Understanding} &$0.80$   &$0.63$   &$0.78$   &$0.61$ \\
    \hline
    \end{tabular}
    \caption{The average scores obtained by \algo when participants were asked to answer three evaluation questions comparing \algo to an assigned baseline. The questions were: ``Which agent cooperates more fluently?'' (abbreviated as fluency), ``Which agent did you prefer playing with?'' (abbreviated as preference), and ``Which agent did you understand better?'' (abbreviated as understanding). A value closer to 1 indicates a stronger preference for \algo among participants.}
    \label{tab:my_label}
\end{table}

In addition to evaluating individual games, human players are also asked to compare the two AI models they played with. They are required to rate the AI performances across three dimensions: fluency through ``Which agent cooperates more fluently?'', preference by ``Which agent did you prefer playing with?'', and understanding with ``Which agent did you understand better?''. The mean rating scores are presented in Table~\ref{tab:after_game}. 

When the value is greater than 0.5, it indicates that human players prefer \algo over the baseline. Consequently, a value closer to 1 signifies a more pronounced preference for \algo compared to the baseline model. As demonstrated in the table, all values are larger than 0.6, suggesting that human players concur that \algo performs better. 

When compared to the SP method, \algo receives the highest scores (above 0.8), implying that participants believe that \algo will outperform SP with a probability of over 80\%. In comparison to the state-of-the-art MEP, \algo's scores also exceed 0.6. These results indicate that human players agree that \algo performs better in terms of fluency, preference, and understanding compared to the baseline.

In summary, human players concur that \algo is more understandable and contributes significantly to their teamwork, particularly in the three layouts that necessitate cooperation: Forced Coord., Counter Circ., and Asymm Adv.

\section{Conclusion}
\label{sec:con}
\changes{In this study, we introduce graphic-form games and preference graphic-form games as intuitive reformulations of cooperative games. These reformulations can effectively evaluate and pinpoint cooperative incompatibility during the learning process of Zero-Shot Coordination (ZSC) algorithms, thereby further addressing the issue of cooperative incompatibility in zero-shot human-AI coordination.
Additionally, we propose the \framework, designed to iteratively approximate the best response preferred by teammates within the latest population. 
Theoretically, we provide proof that \framework converges towards the locally optimal strategy preferred by the rest of the population. If the in-degree centrality is selected as the preference centrality function, the convergence rate would achieve Q-sublinear status.}

\changes{We also implemented an online pipeline for the Overcooked Human-AI experiment, which allows for easy modifications of questionnaires, model weights, and other elements.
To the best of our knowledge, it is the first comprehensive human-AI experimentation pipeline for zero-shot human-AI coordination evaluation including turnkey experimental procedures and scale design.
}
Through the pipeline, we engaged 130 participants in human experiments, and the results highlighted a general preference for our approach over SOTA methods across various subjective metrics.
Furthermore, our objective experiments in the Overcooked environment demonstrated that our algorithm, referred to as \algo, exceeded the performance of SOTA algorithms when coordinating with new AI agents and the human proxy model. It also demonstrated the efficient resolution of cooperative incompatibility.

\textbf{Limitations.} Convergence of \framework requires satisfying the assumption that the preference centrality of the newly generated strategy is ranked in the top $k$, which is controlled by an additional hyperparameter.
If $k$ is too big with a lower ranking, the \framework will slowly converge to the best-preferred strategy. 
On the other hand, if $k$ is too small and refers to a higher ranking, each generation's assumption will not be guaranteed, which will easily cause learning failure.
Besides, in our implemented algorithm~\algo, we introduce the Shapley Value as the tool and develop the Graphic Shapley Value to analyze the cooperative ability.
Although we have utilized Monte Carlo permutation sampling to reduce the computational complexity, the computational complexity is still high.
Therefore, we only maintain a population of 50 for the limitation of computational resources.

\textbf{Future Work.} Future work will focus on performing an adaptive mechanism that automatically selects a suitable value for the hyperparameter $k$ to improve the convergence rate without promoting iterations of each update. 
Meanwhile, improving the efficiency of the graphic Shapley Value solver and exploring other cooperative ability evaluation solvers are important in developing the framework.
Future work also includes the development of practical algorithms for more complex games except for Overcooked.

\clearpage
\acks{Yang Li is supported by the China Scholarship Council (CSC) Scholarship.
The Shanghai Jiao Tong University team is supported by National Natural Science Foundation of China (No.62106141) and Shanghai Sailing Program (21YF1421900).
The authors thank Xihuai Wang for his kind assistance and advice.
We also thank Jia Guo and Tao Shi for their support for our Human-AI experiment platform development.
}

\appendix
\appendix

\section{Proofs of Theorem~\ref{thm: converge}}
\label{appendix:proofs_thm}
\FirstTHM*
\begin{proof}

According to the definition of the local best-preferred strategy, the local optimal strategy is the node with zero preference centrality ($\eta$). Therefore, we need to prove that the value of $\eta$ will approach zero.

\changes{
Let $\eta_t$ denote the centrality value of the preference of the updated strategy $s_t$ in generation $t$, where $0\leq \eta \leq 1$. 
We first remark on the RL oracle defined in Eq.~\ref{eq:oracle_approx}:
$
    s_{n+1} = \operatorname{oracle}(s_n, \gJ(s_n, \phi)),
\ with\ 
\mathcal{R}(\eta(s_{n+1}))>k.
$
Under the assumption of the approximated oracle functioning effectively, it follows that the preference centrality $\eta_t$ associated with generated strategy $s_t$ at generation $t$ resides among the first $k$ positions when arranged in ascending order of centrality values.
For simplicity, We define a group $g_t$ at generation $t$ as the set of strategies with the lowest $k$ preference centrality values. 
\begin{lemma}
\label{lemma:group}
    Provided that the approximated oracle is functioning effectively, the maximal preference centrality value, denoted as $\eta_{g_t}$, within group $g_t$ is expected to exhibit a diminishing trend with each successive generation. 
    \end{lemma}
\begin{proof}
    In light of the approximated oracle's definition, the preference centrality value of the strategy $s_t$, generated at generation $t$, will occupy one of the first $k$ positions when sorted in ascending order based on preference centrality values. Consequently, the initial $k$ strategies of group $g_t$ will undergo an update, in which the strategy ranking $k$-th with the highest preference centrality value is substituted by either the $(k-1)$-th strategy or the newly generation strategy, $s_t$. Under any given conditions, the maximum preferential value within the group will experience a reduction.
\end{proof}
}

\changes{
Let $\eta_{g_t}$ denote the largest preference centrality in the group $g$.
Thus, we can derive the subsequent equation based on Lemma~\ref{lemma:group}. 
\begin{equation} 
\eta_{g_t} = \eta_{g_{t-1}} - \epsilon_{t-1}, 
\end{equation} 
where $\epsilon_{t-1}$ is a positive value and $0< \epsilon \leq \eta_{g_{t-1}}$.
By further simplifying the equation, we have
\begin{equation} 
\begin{aligned} 
\eta_{g_{t}} &= \eta_{g_{t-1}} - \epsilon_{t-1},\\ &=\eta_{g_{t-1}} -\alpha_{t-1} \eta_{g_{t-1}},\\ &=\beta_{t-1} \eta_{g_{t-1}}, \end{aligned} 
\end{equation} 
where the second line employs $\eta_{g_{t-1}}$ to substitute the residual term, with the adjustment parameters $0 < \alpha_{t-1} \leq 1$ and $\beta_{t-1} = 1 - \alpha_{t-1}.$}

\changes{
Assuming that the centrality value of the preference in the initial step is $0\leq \eta_0 \leq 1$, we can recursively calculate the following formula: 
\begin{equation} \begin{aligned} \eta_{g_t} &=\beta_{t-1} \eta_{g_{t-1}},\\ &=\beta_{t-1} \beta_{t-2} \eta_{g_{t-2}},\\ &=\cdots, \\ &=\prod_{i=0}^{t-1} \beta_i \times \eta_{g_0}. 
\end{aligned} 
\label{eq:iter_eta} \end{equation} For any $\beta \in {\beta_0, \cdots, \beta_{t-1}}$, we have $1>\beta\geq 0$. In addition, we set $\beta_t$ as a very small positive number if $\eta_t=0$.
Furthermore, we ascertain that the coefficient $\beta$ is not consistently zero. This is due to the fact that $\beta = 0$ would imply a preference centrality of zero for the strategy, which is not universally attainable within the context of the RL oracle. This very limitation underpins our rationale for introducing the approximated RL oracle.
Thus, we can conclude that $\eta_t$ will approach zero within the population as outlined in~\eqref{eq:iter_eta}.} 

\changes{
Through this proof, we have substantiated that under the effective functioning of the RL oracle as characterized by Eq.~\ref{eq:oracle_approx}, the sequence ${s_i}$ is progressing towards the zero of preference centrality within the population. That is, the sequence is converging to a strategy denoted by $s^*$, which represents a locally best-preferred solution.
}
\end{proof}

\section{Proof of Corollary~\ref{lemma: converge_rate}}
\label{appendix:proofs_corollary}
\FirstLEMMA*
\begin{proof}
    In Theorem~\ref{thm: converge}, we have proved that the strategies generated by the \framework~will converge to the local best-preferred strategy.
When we use the in-degree centrality function as $\eta$, the preference centrality function can be rewritten as:
\begin{equation}
        \eta(i) = 1-\frac{I_i}{n-1},
    \end{equation}
where $I_i$ is the in-degree of node $i$ and $n$ is the size of the strategy set $\gN$.

\changes{
    Therefore, we have
    \begin{equation}
\begin{aligned}
\label{eq:proof_1}
&\lim\limits_{t \to \infty} \frac{|\eta_{t+1} - 0|}{|\eta_{t} - 0|} \\
        = &\lim\limits_{t \to \infty} \frac{\eta_{t+1}}{\eta_{t}} \\
        =& \lim\limits_{t \to \infty} \frac{1-\frac{I_{t+1}}{t}}{1-\frac{I_{t}}{t-1}} \\
        = &\lim\limits_{t \to \infty} \frac{t-1}{t} \frac{t - I_{t+1}}{t-I_t-1} \\
        =&\lim\limits_{t \to \infty} \frac{t - I_{t+1}}{t-I_t-1} \\
        =&1
\end{aligned}
\end{equation}
}

Therefore, using the in-degree centrality, we can conclude that the \framework~will converge to the local optimal strategy at a Q-sublinear rate.
\end{proof}

\section{Experimental Details of \algo}
\label{appendix:cole}
This paper utilizes Proximal Policy Optimization (PPO)~\cite{PPO} as the oracle algorithm for our set of strategies $\gN$, which consists of convolutional neural network parameterized strategies. Each network is composed of 3 convolution layers with 25 filters and 3 fully-connected layers with 64 hidden neurons. To manage computational resources, we maintain a population size of 50 strategies. In instances where the population exceeds this limit, we randomly select one of the earliest ten removal strategies.

We run and evaluate all our experiments on Linux servers, which include two types of nodes: 1) 1-GPU node with NVIDIA GeForce 3090Ti 24G as GPU and AMD EPYC 7H12 64-Core Processor as CPU, 2) 2-GPUs node with GeForce RTX 3090 24G as GPU and AMD Ryzen Threadripper 3970X 32-Core Processor as CPU.
On the Overcooked game environment, \algo takes one to two days on the 2-GPUs machine for one layout's training.

The hyperparameter setup is similar to those in PBT and MEP, which are given as follows. 
\begin{itemize}
    \item The learning rate for each layout is  2e-3 , 1e-3 , 6e-4 , 8e-4 , and 8e-4.
    \item The gamma $\gamma$ is 0.99.
    \item The lambda $\lambda$ is 0.98.
    \item The PPO clipping factor is 0.05.
    \item The VF coefficient is 0.5.
    \item The maximum gradient norm is 0.1.
    \item The total training time steps for each PPO update is 48000, divided into 10 mini-batches.
    \item The total numbers of generations for each layout are 80, 60, 75, 70, and 70, respectively.
    \item For each generation, we update 10 times to approximate the best-preferred strategy.
    \item The $\alpha$ is 1.
\end{itemize}

\section{Implementations of Baselines}
\label{appendix:base}
In this part, we will introduce the detailed implementations of baselines.
We train and evaluate self-play and PBT based on the Human-Aware Reinforcement Learning repository\footnote{\url{https://github.com/HumanCompatibleAI/human_aware_rl/tree/neurips2019}.}~\cite{HARL}  and used Proximal Policy Optimization (PPO)~\cite{PPO} as the RL algorithm.
We implement FCP according to the FCP paper~\cite{FCP} and use PPO as the RL algorithm.
The implementation is based on the Human-Aware Reinforcement Learning repository (the same used in the self-paly and PBT).
The MEP agent is trained with population size as 5, following the MEP paper~\cite{MEP} and used the original implementation\footnote{The code of MEP original implementation: \url{https://github.com/ruizhaogit/maximum_entropy_population_based_training}.}.

\begin{figure}[ht!] \centering    \includegraphics[width=0.98\linewidth]{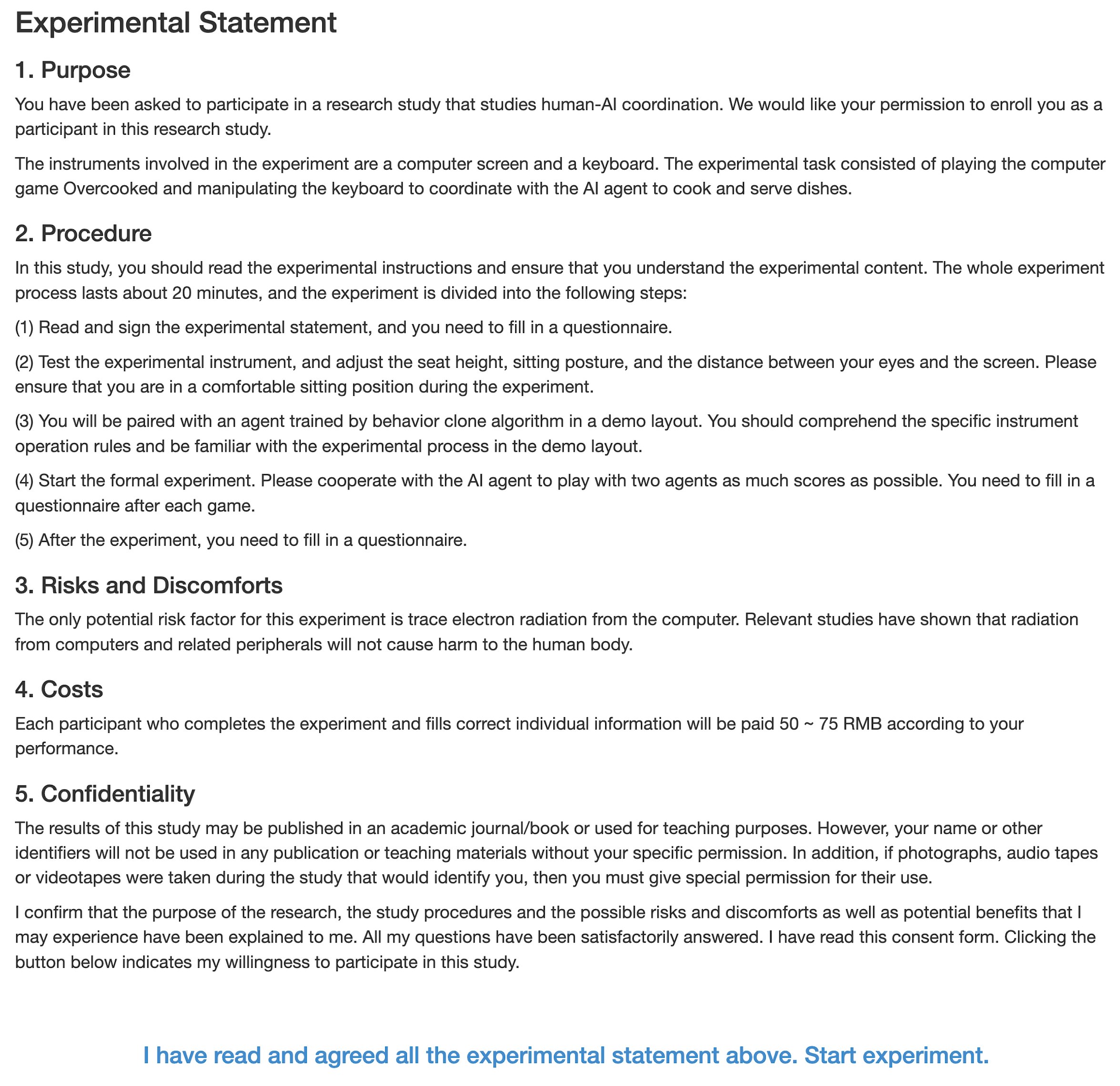}
\caption{Screenshots of the Human-AI Experiment Platform - Experiment Statement. }  
\label{fig:platform_state}     
\end{figure}

\begin{figure}[ht!] \centering    \includegraphics[width=0.98\linewidth]{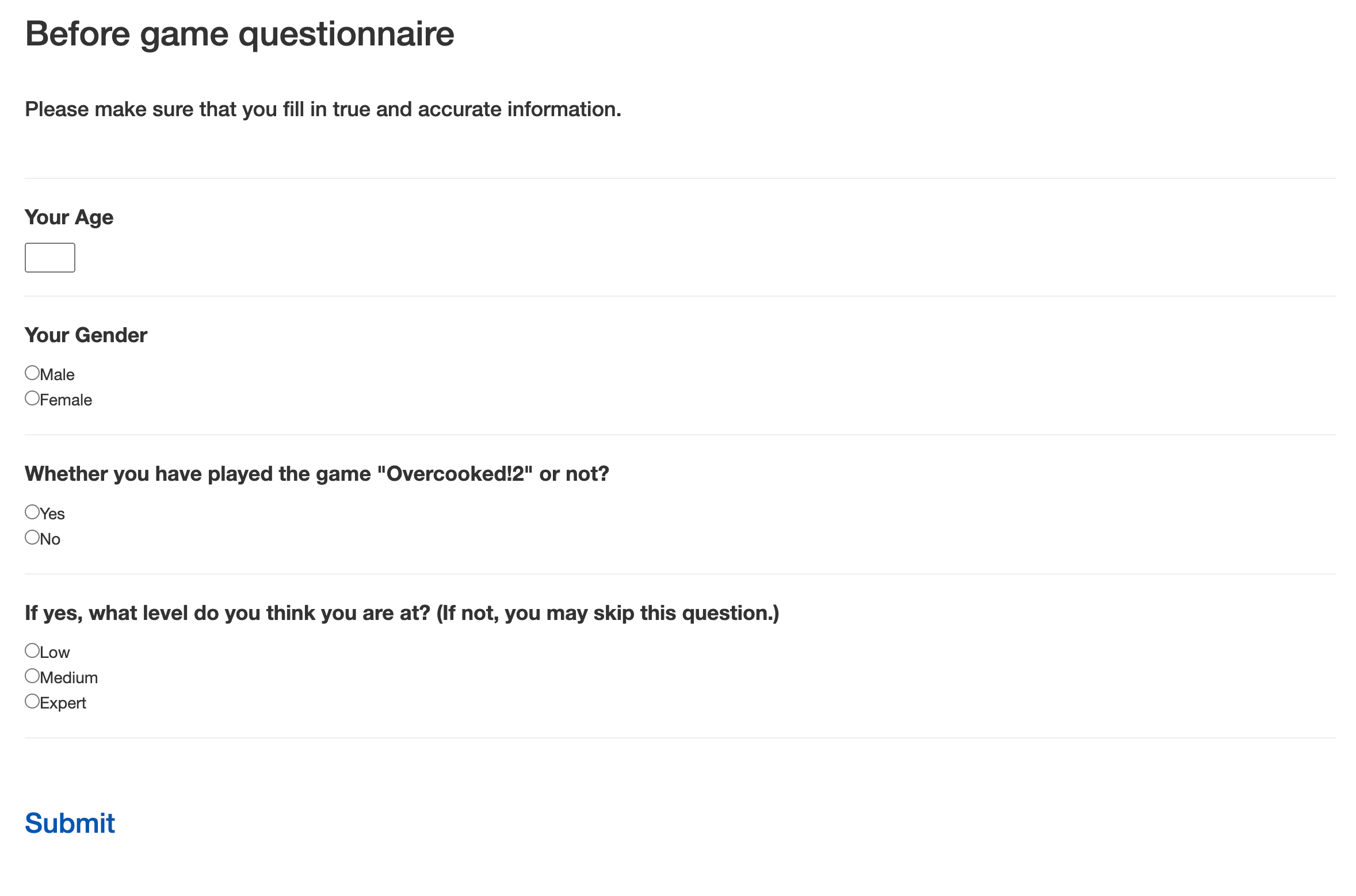}
\caption{Screenshots of the Human-AI Experiment Platform - participant information questionnaire.}  
\label{fig:platform_before}     
\end{figure}

\begin{figure}[ht!] \centering    \includegraphics[width=0.98\linewidth]{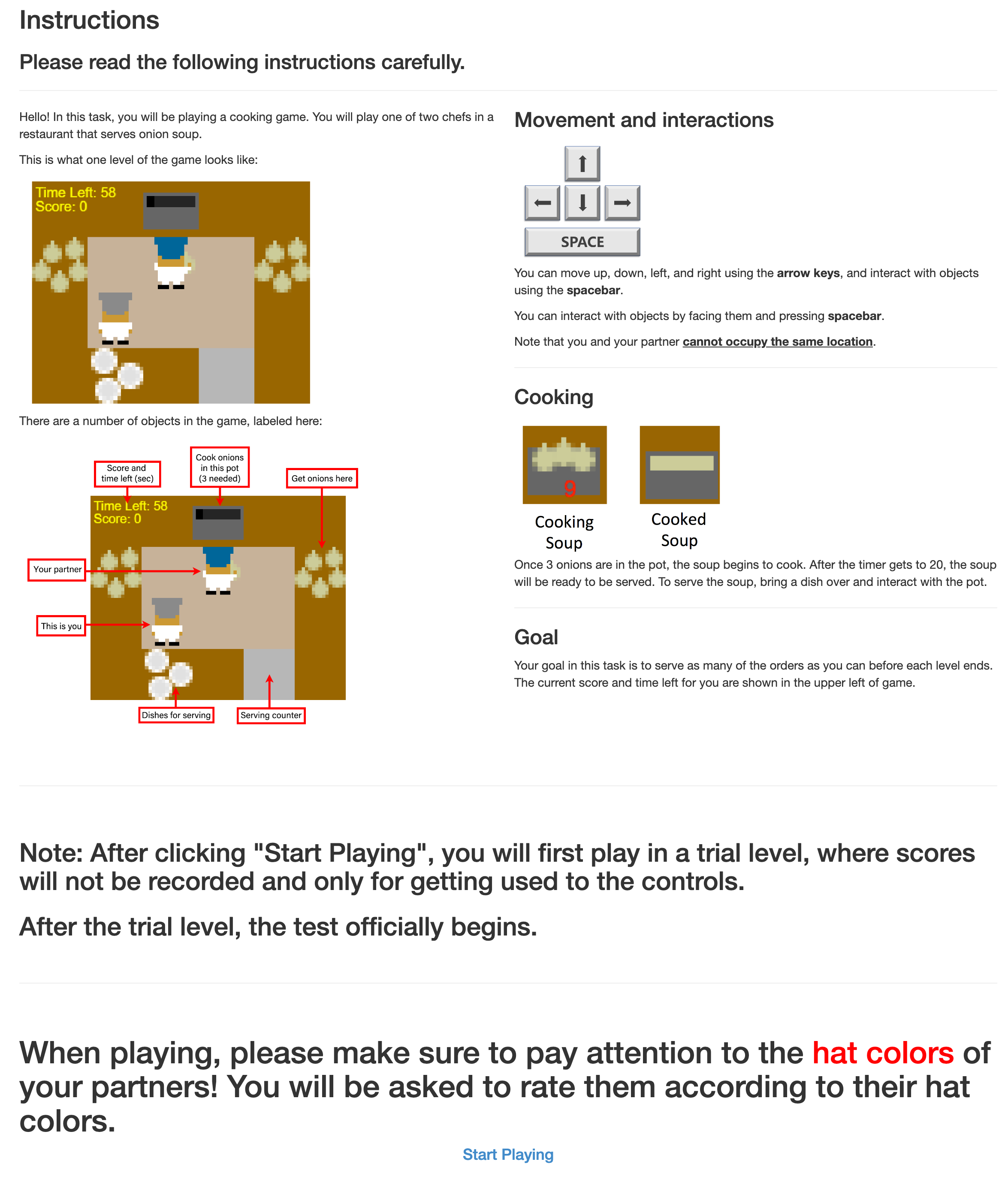}
\caption{Screenshots of the Human-AI Experiment Platform - instruction providing an overview of the environment interface, operation instructions, and game objectives.}  
\label{fig:platform_instruction}     
\end{figure}

\begin{figure}[ht!] \centering    \includegraphics[width=\linewidth]{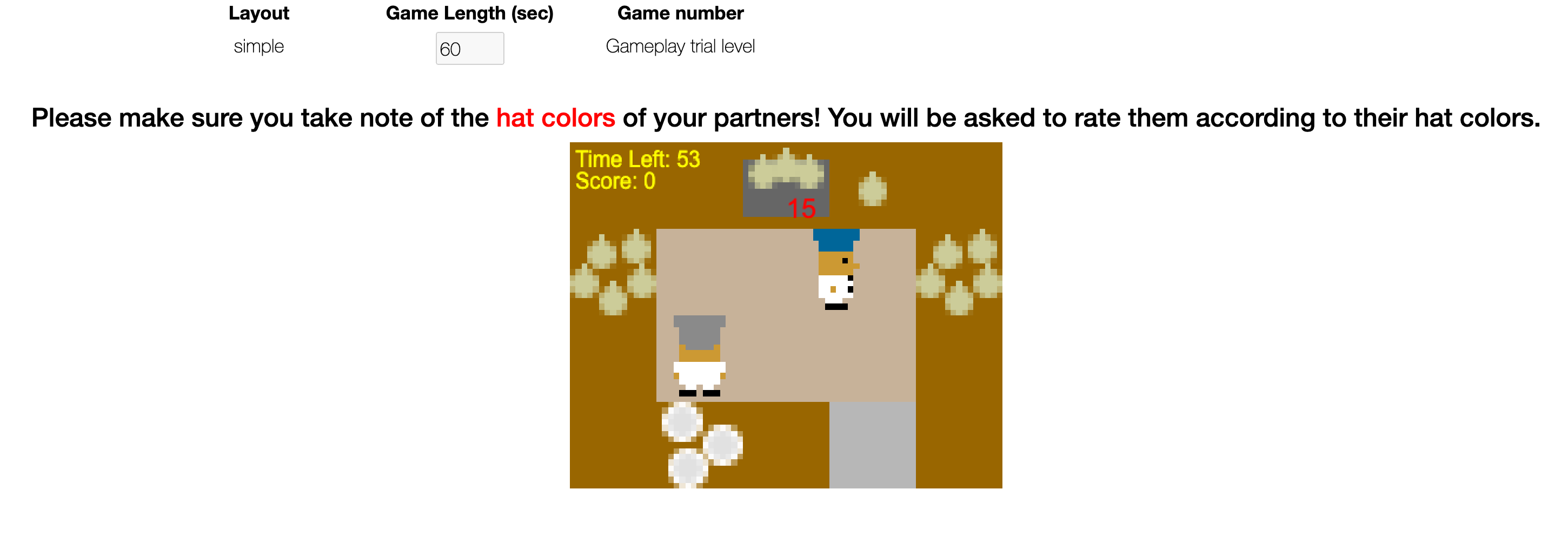}
\caption{Screenshots of the Human-AI Experiment Platform - trial playing: players engage with the human proxy model to familiarize themselves with the system.}  
\label{fig:platform_trial}     
\end{figure}

\begin{figure}[ht!] \centering    \includegraphics[width=\linewidth]{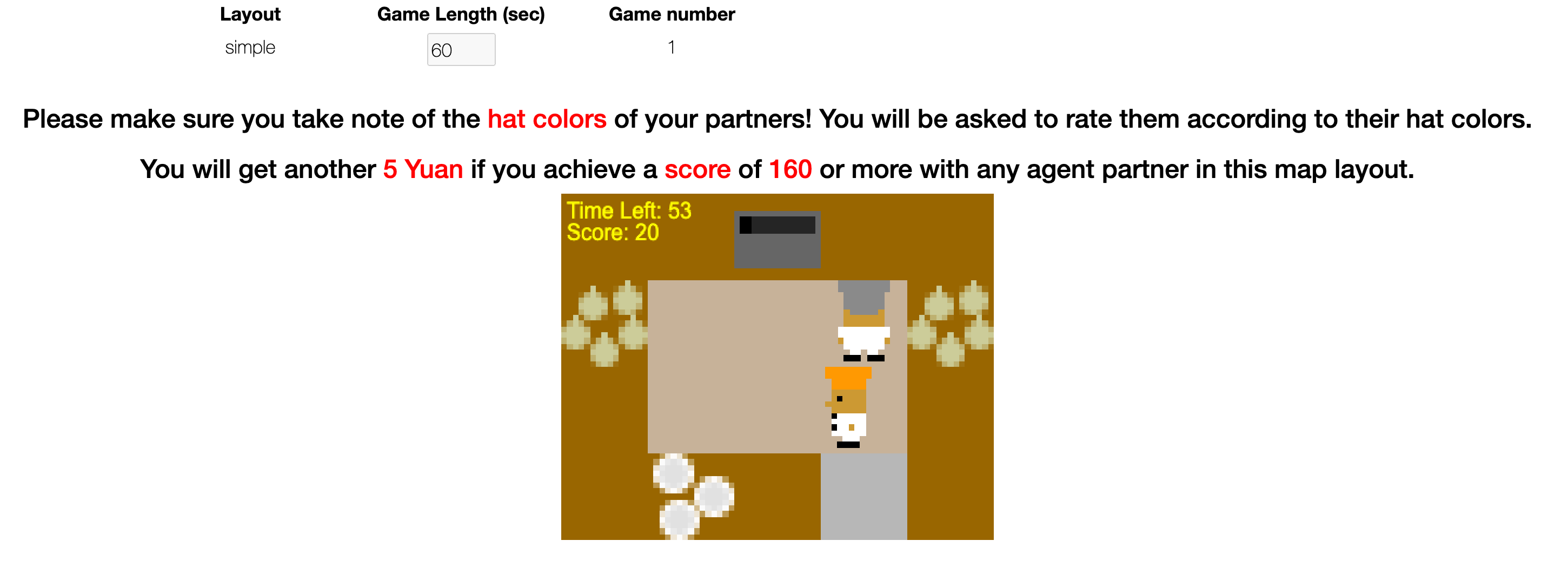}
\caption{Screenshots of the Human-AI Experiment Platform - Game Playing: displaying the layout name, game length (in seconds), and game number at the top, followed by reminders about hat colors and the bonus awarded for the current game. The middle area is the game interface, where human players interact with the AI agent.}  
\label{fig:platform_game}     
\end{figure}

\begin{figure}[ht!] \centering    \includegraphics[width=0.9\linewidth]{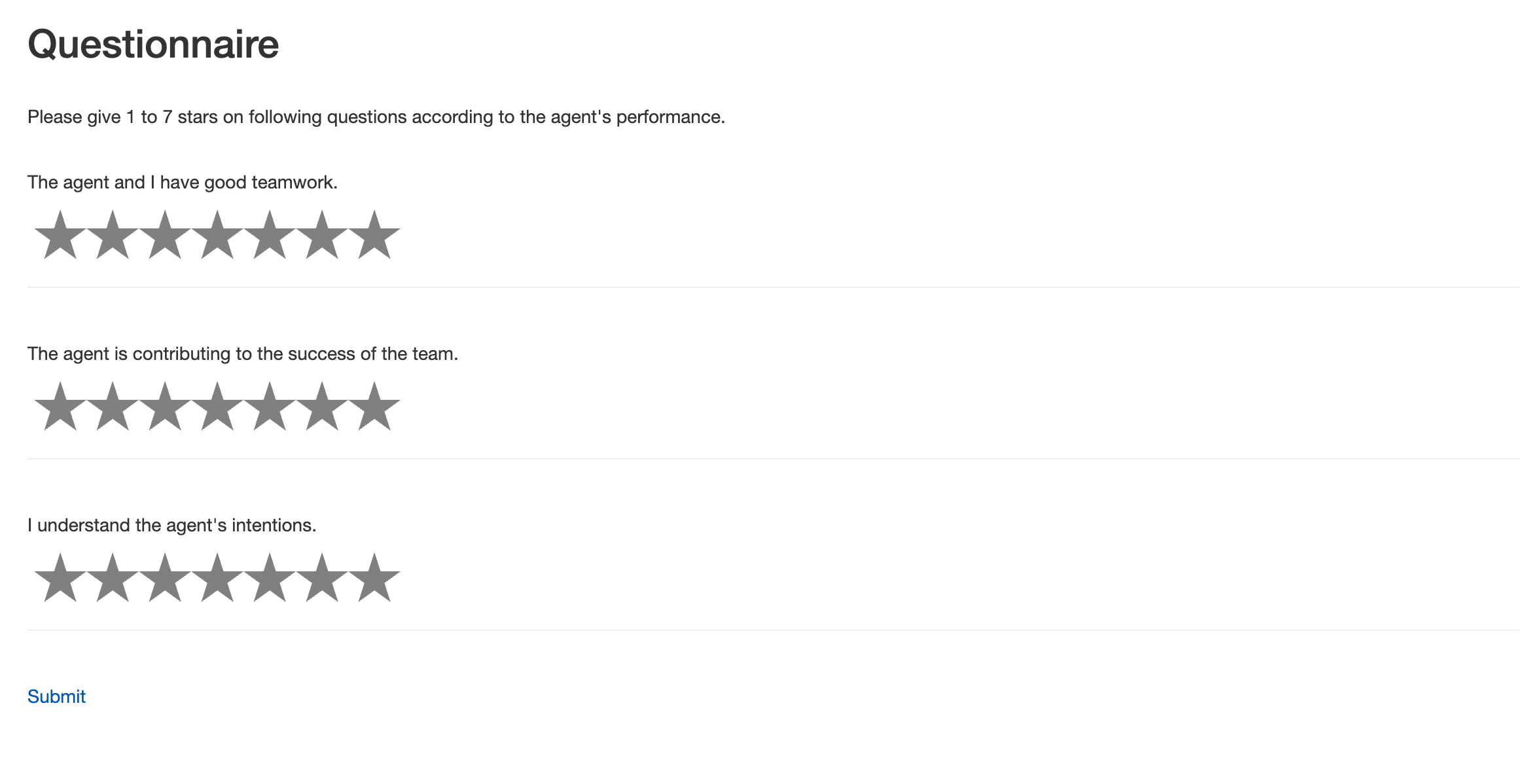}
\caption{Screenshots of the Human-AI Experiment Platform - Individual Questionnaire: Participants score the performance of the AI teammate in the finished game.}  
\label{fig:platform_in}     
\end{figure}

\begin{figure}[ht!] \centering    \includegraphics[width=0.9\linewidth]{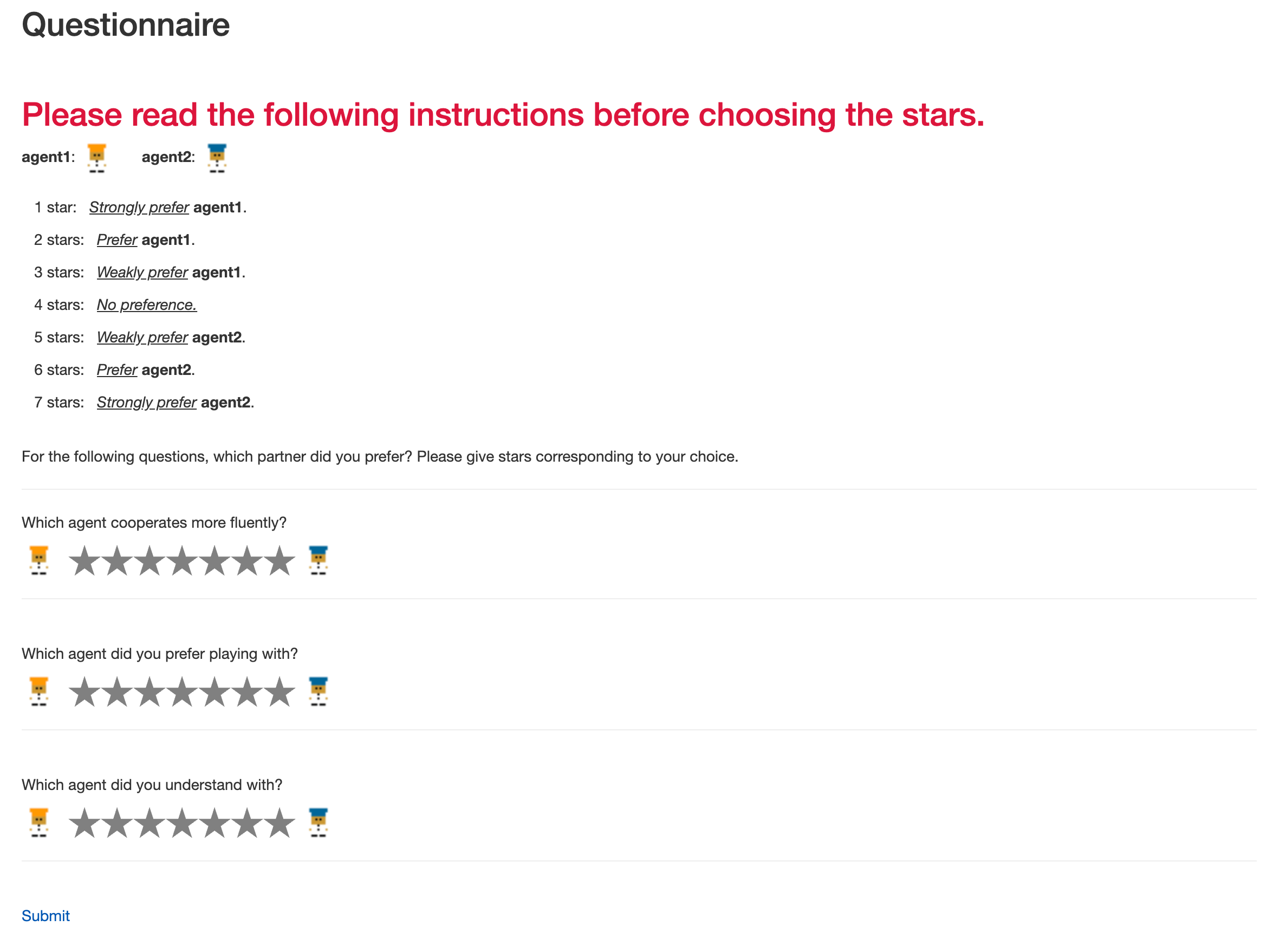}
\caption{Screenshots of the Human-AI Experiment Platform - Final Questionnaire: Participants compare and rank the performance of AI partners.}  
\label{fig:platform_after}     
\end{figure}

\section{Visual Overview of the Human-AI Experiment Platform}
\label{app:screenshot}
In this section, we provide a visual representation of the Human-AI Experiment Platform to offer readers a better understanding of the system's user interface, layout, and functionality. 
By including screenshots of various stages of the experiment process in Fig.~\ref{fig:platform_state}, Fig.~\ref{fig:platform_before}, Fig.~\ref{fig:platform_instruction}, Fig.~\ref{fig:platform_trial}, Fig.~\ref{fig:platform_game}, Fig.~\ref{fig:platform_in}, and Fig.~\ref{fig:platform_after},  we aim to provide a comprehensive overview of the platform, enabling readers to better grasp its design and implementation.

\clearpage


\bibliography{reference}
\bibliographystyle{theapa}

\end{document}